\documentclass{article}

\PassOptionsToPackage{numbers, compress}{natbib}
\PassOptionsToPackage{svgnames}{xcolor}

\usepackage[main,final]{neurips_2026}

\usepackage[utf8]{inputenc} 
\usepackage[T1]{fontenc}    
\usepackage{hyperref}       
\usepackage{url}            
\usepackage{booktabs}       
\usepackage{amsfonts}       
\usepackage{microtype}      
\usepackage[table]{xcolor}

\usepackage{graphicx}

\usepackage{multirow}
\usepackage{threeparttable}
\usepackage{tabularx}

\usepackage{amsmath}
\usepackage{amssymb}
\usepackage{amsthm}
\usepackage{bbm}

\usepackage{soul}
\usepackage{wrapfig}

\usepackage[capitalize,noabbrev]{cleveref}

\usepackage[textsize=tiny]{todonotes}
\usepackage[printonlyused]{acronym}

\newcolumntype{Y}{>{\centering\arraybackslash}X}

\definecolor{bloatcolor}{RGB}{255, 220, 220}
\newcommand{\bloat}[1]{\sethlcolor{bloatcolor}\hl{#1}}
\newcommand{\bridge}[1]{\textbf{\textcolor{blue!60!black}{#1}}}

\newcommand{\paradox}{Attention–Markov Incompatibility}

\usepackage{amsmath,amsfonts,bm}

\def\eqref#1{eq.~\ref{#1}}

\def\1{\bm{1}}

\DeclareMathAlphabet{\mathsfit}{\encodingdefault}{\sfdefault}{m}{sl}
\SetMathAlphabet{\mathsfit}{bold}{\encodingdefault}{\sfdefault}{bx}{n}

\def\gL{{\mathcal{L}}}

\def\gX{{\mathcal{X}}}

\usepackage{amsthm}
\usepackage{amssymb}
\newtheorem{theorem}{Theorem}[section]

\theoremstyle{definition}

\newtheorem{proposition}[theorem]{Proposition}

\theoremstyle{remark}

\newcommand*{\prob}[1]{\mathbb{P}}

\newcommand{\E}{\mathbb{E}}

\usepackage{empheq}
\usepackage{fancybox}
\usepackage{tcolorbox}

\theoremstyle{definition}

\acrodef{CoT}{Chain-of-Thought}
\acrodef{IB}{Information Bottleneck}
\acrodef{CIB}{Conditional Information Bottleneck}
\acrodef{RL}{Reinforcement Learning}
\acrodef{LLM}{Large Language Model}
\acrodef{LLMs}{Large Language Models}
\acrodef{MI}{Mutual Information}
\acrodef{SOTA}{state-of-the-art}
\acrodef{PRM}{Process Reward Model}

\definecolor{dlercib15}{HTML}{1F77B4}   
\definecolor{dlercib7}{HTML}{FF7F0E}    
\definecolor{deepcib15}{HTML}{2CA02C}   
\definecolor{deepcib7}{HTML}{D62728}    
\definecolor{cibgray}{RGB}{245,245,245} 

\title{Reasoning as Compression: Unifying Budget Forcing via the Conditional Information Bottleneck}

\author{%
  Fabio Valerio Massoli\thanks{Corresponding author: \texttt{fmassoli@qti.qualcomm.com}} \\
  Qualcomm AI Research$^\dagger$ \\
  \And
  Andrey Kuzmin \\
  Qualcomm AI Research$^\dagger$ \\
  \And
  Arash Behboodi \\
  Qualcomm AI Research\thanks{Qualcomm AI Research is an initiative of Qualcomm Technologies, Inc.}
}

\begin{document}

\maketitle

\begin{abstract}


\ac{CoT} prompting improves LLM accuracy on complex tasks but often increases token usage and inference cost. Existing ``Budget Forcing'' methods reduce cost via fine-tuning with heuristic length penalties, suppressing both essential reasoning and redundant filler. We recast efficient reasoning as a lossy compression problem under the \ac{IB} principle, and identify a key theoretical gap when applying naive \ac{IB} to transformers: attention violates the Markov property between prompt, reasoning trace, and response. To resolve this issue, we model \ac{CoT} generation under the \ac{CIB} principle, where the reasoning trace $Z$ acts as a computational bridge that contains only the information about the response $Y$ that is not directly accessible from the prompt $X$. This yields a general Reinforcement Learning objective: maximize task reward while compressing completions under a prior over reasoning traces, subsuming common heuristics (e.g., length penalties) as special cases (e.g., uniform priors). In contrast to naive token-counting approaches, we introduce a semantic prior that measures token cost by surprisal under a language model. 
Crucially, the prior is queried only for token-level log-probabilities, adding negligible overhead to the training loop.
Empirically, our \ac{CIB} objective prunes reasoning redundancy while preserving fluency and logic, improving accuracy at moderate compression and enabling aggressive compression with minimal accuracy drop. These gains generalize across model families and task domains, confirming \ac{CIB} as a domain-agnostic CoT compression framework.

\end{abstract}

\section{Introduction}\label{sec:intro}

\begin{figure}[!ht]
  \centering
    \includegraphics[width=\linewidth]{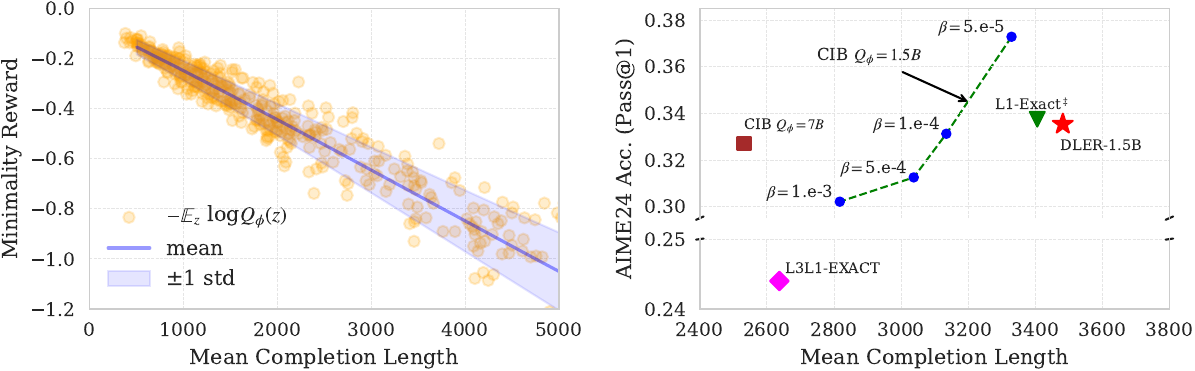}
    \caption{
    \textbf{Left: Minimality Reward vs. Completion Length.} We observe a consistent negative correlation between the completion length and the minimality reward. The shaded blue region shows the $\pm1\sigma$ band representing the spread of the information cost for the token chosen within CoTs with similar length. 
  \textbf{Right: Accuracy vs. Completion Length.} 
  The $\beta$ parameter confers fine-grained control over the accuracy-compression trade-off. 7B prior ({\color{brown}$\blacksquare$}) yields stronger compression than a 1.5B prior ({\color{blue}$\bullet$}). Baselines: DLER~\citep{liu2025dler} ({\color{red}$\bigstar$}), L3L1-EXACT~\citep{aggarwal2025l1} ({\color{Magenta}\rotatebox{45}{$\blacksquare$}}), and our L1-Exact$^\ddagger$ model ({\color{Green}$\blacktriangledown$}).
    }
    \label{fig:side_side_1}
    \vspace{-5pt}
\end{figure}


Chain-of-Thought (CoT) prompting \citep{wei2022chain} is the primary mechanism for unlocking reasoning in \ac{LLMs}, allowing models to allocate test-time computation for complex tasks. However, this gain incurs significant costs: reasoning chains are often excessively verbose, increasing latency and compute usage. Consequently, ``Budget Forcing''---constraining models to yield correct answers within a restricted token budget---has emerged as a critical frontier in efficient inference. Current approaches relying on naive length penalties or strict training-time length constraints are suboptimal. Whether penalizing output length or enforcing a hard token limit, these methods impose a uniform cost on every token, implicitly assuming all tokens contribute equally to the solution. This uniform token penalty---effectively a ``flat tax'' on 
computation---ignores the distinction between essential reasoning steps and redundant fillers. Optimizing under such a metric is brittle: models are incentivized to delete tokens regardless of semantic relevance, discarding crucial intermediate logic to satisfy the budget. This makes the accuracy--compute trade-off difficult to tune, as a single weight (or limit) may over-penalize hard prompts while under-penalizing redundancy in easy ones. \\
In this work, we reframe efficient reasoning not as token minimization, but as \emph{lossy compression}. We propose a unified framework based on the Information Bottleneck (IB) principle \citep{tishby1999information}, positing that an ideal reasoning chain is the minimal sufficient statistic of the prompt required to predict the answer. We identify that standard \ac{IB} \citep{tishby1999information} cannot be naively applied to transformers due to a theoretical inconsistency we term the ``\emph{\paradox}'': the attention mechanism grants the decoder direct access to the prompt, violating the Markov chain assumption ($Y \leftrightarrow X \leftrightarrow Z$) required by standard IB. 
We resolve the problem by modeling \ac{CoT} generation under the
Conditional Information Bottleneck (CIB) as \emph{Source Coding with Side Information}. As a result, a novel \ac{RL} objective naturally arises from the \ac{CIB} framework. Instead of a uniform length penalty, we assign a \emph{semantic cost} to each token based on its information content relative to a frozen base model. This formulation aligns cost with information flow: the model is encouraged to ``pay'' for informative tokens that increase answer probability while suppressing redundancy. Empirically, this allows for precise navigation of the Pareto frontier, achieving a superior accuracy--compression trade-off compared to length-based baselines (Figure~\ref{fig:side_side_1}, right).\\
Our contributions are as follows:

\begin{itemize}
\item We identify the limitations of length-based budget forcing, showing that uniform penalties and hard limits conflate essential reasoning with redundancy.
\item We propose a theoretical framework resolving the ``\paradox'' via the \ac{CIB}, yielding a semantic token cost based on relevance rather than length.
\item We demonstrate that this formulation compresses reasoning traces while achieving Pareto optimal accuracy-compression trade-off.
\end{itemize}

The remainder of the paper is organized as follows. Section~\ref{sec:related} reviews related work.  Section~\ref{sec:methodology} formalizes the ``\paradox", derives the \ac{CIB} objective, and presents our reward model. Section~\ref{sec:theory} proves that standard length-based penalties arise as a special case of \ac{CIB} under non-informative priors. Section~\ref{sec:exp_res} reports experimental results. Section~\ref{sec:conclusions} we report our conclusions and briefly discuss limitations.

\section{Related Work} \label{sec:related}

\paragraph{Budget Forcing and Efficient Reasoning.}
Recent studies suggest that optimal reasoning compute should scale with problem complexity \citep{zhang2025laws}, yet unconstrained models often exhibit excessive verbosity even on simple tasks \citep{muennighoff2025s1}. This has motivated ``Budget Forcing'' strategies spanning training and inference, including reward shaping with length costs \citep{aggarwal2025l1}, token-level advantages~\citep{he2026iapo}, and hard truncation~\citep{liu2025dler}. More granular approaches include difficulty-aware allocation \citep{cheng2025optimizing} and reference-guided budgeting \citep{wu2025lapo,li2025selfbudgeter,luo2025o1}, monotonic uncertainty reduction~\citep{wei2026infodensity}, sometimes tracking history \cite{huang2025hapo} or decomposing costs per-token \citep{jiang2025overthinking}. Inference-only methods steer generation via auxiliary predictors \cite{li2025budgetguidance,han2025tale} or employ early-exit decoding \citep{mao2025escot,wang2025eat}. Alternative paradigms replace verbose \ac{CoT} with concise drafting \citep{xu2025cod,renze2024ccot}, selective reasoning policies \citep{wang2025ton}, or trace compression via token pruning and skipping \cite{xia2025tokenskip,choi2025caccot,cui2025perplexity,cheng2024compressedcot}. \cite{wang2024tokeneconomies} further propose budget-aware evaluation metrics. Several more recent works consider budget forcing in the context of multi-turn reasoning tasks~\citep{jali2026not,sethi2026don, ma2026odar}. While effective, these methods largely rely on naive token counts as a cost proxy. In contrast, we ground budget forcing in information theory, penalizing tokens based on semantic surprisal rather than raw length.

\paragraph{Information Theory in Large Language Models.} 
The \ac{IB} principle \citep{tishby1999information} was proposed as a framework for analyzing deep learning \citep{Shwartz-Ziv2017-IB}, followed by various discussions \citep{Saxe2018-IB}, applications in reasoning and robustness \citep{huang2025ibro}, and hallucination detection \citep{wang2024understanding}. However, these works differ from ours in two key respects. First, their objectives typically target generalization or explainability of deep learning rather than strict computational efficiency of reasoning models. Second, they apply the standard \ac{IB} formulation, which assumes a Markov chain where the latent representation $Z$ mediates all information. Instead, we explicitly take into account the structure of transformer architectures,  where the attention mechanism grants the decoder direct access to the prompt, $X$, creating a collider structure $(X, Z) \to Y$ which breaks the aforementioned Markov property.  To the best of our knowledge, this work is the first to unify ``Budget Forcing'' and Information Theory under a CIB framework. 

\section{Methodology} \label{sec:methodology}
\begin{figure}[t]
    \centering
    \includegraphics[width=\linewidth]{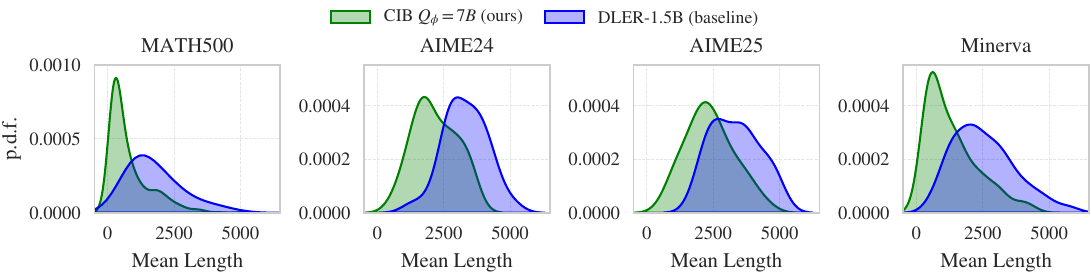}
    \caption{\textbf{Lengths Distribution.} Compared to the baseline length distribution (\textcolor{blue}{blue curve}), the minimality term shifts the length distribution towards shorter completions (\textcolor{DarkGreen}{green curve}). The plotted distributions correspond to models with similar accuracy (see Table~\ref{tab:results}).}
    \label{fig:lengths_dist}
    \vspace{-5pt}
\end{figure}

We now formalize efficient reasoning as an optimization problem within the CIB framework. We report here the most relevant equations and we refer the reader to Appendix~\ref{app:CIB} for the full derivation. In what follows, we refer to $X$, $Z$, and $Y$, as the prompt, \ac{CoT}, and ground truth answer, respectively.
\paragraph{The \paradox.} 
    The standard Information Bottleneck (IB) principle \citep{tishby1999information} seeks a representation $Z$ that maximally compresses the input $X$ while preserving information about the target $Y$. Formally, it minimizes the Lagrangian:\begin{equation}\mathcal{L}_{\text{IB}} = I(X; Z) - \mu I(Y; Z)\end{equation} over $P(Z|X)$ where $\mu$ controls the trade-off between compression (minimizing mutual information $I(X; Z)$) and prediction (maximizing $I(Y; Z)$). Crucially, the standard IB assumes the Markov chain $Y \leftrightarrow X \leftrightarrow Z$, implying that $Z$ is the sole channel through which information flows from $X$ to $Y$. However, this assumption is fundamentally violated in transformer-based \ac{LLMs}. Due to the causal attention mechanism, the decoder predicting $Y$ attends to \emph{both} the prompt $X$ and the generated chain $Z$. This forms a collider structure: $(X, Z) \to Y$. We term this inconsistency the \emph{\paradox}. 
Under the standard IB objective, maximizing $I(Y;Z)$ can be inefficient as it ignores that the model has access to the query $X$ during the answer generation. This can lead to keeping redundant information about the query $X$. It is important to note that the conditional probability $P(Y|X)$ of the answer given the query is unknown, and exactly what we want to \textit{simulate} using the intermediate reasoning trace $Z$.

\paragraph{Conditional Information Bottleneck for Reasoning.}
To resolve the \paradox, we propose grounding ``Budget Forcing'' in the Conditional Information Bottleneck (CIB). We view the prompt $X$ as \emph{side information} that is always available to the answer generator. We require $Z$ to encode only the \emph{additional} information necessary to predict $Y$ given $X$. The objective becomes:\begin{equation}\mathcal{L}_{\text{CIB}} = I(X; Z) - \mu I(Y; Z|X)\end{equation} Minimizing $I(X; Z)$ (or a related upper bound on the rate) while maximizing the conditional predictive power $I(Y; Z | X)$ ensures that the chain $Z$ is penalized for redundancy with $X$ but rewarded for explaining $Y$. We use an LLM $\pi_\theta(\cdot|\cdot)$ to re-parameterize the optimization problem. 

\subsection{Problem Formulation}
We consider a reasoning task defined by a dataset distribution $P_{\mathcal{D}}(X, Y)$, where $X$ is the prompt and $Y$ the ground truth answer. We aim to learn a stochastic policy $\pi_\theta(Z \mid X)$ that generates a \ac{CoT} $Z$ to bridge the gap between $X$ and $Y$, while $\pi_\theta(Y \mid X, Z)$ generates the correct answer. \\
Our goal is to optimize the policy $\pi_\theta$ to maximize the \textbf{Sufficiency} of $Z$ for predicting $Y$, while minimizing the \textbf{Minimality} (information cost) of $Z$ relative to the side information $X$. This is formalized by the \ac{CIB} objective: $\min_{\theta} \mathcal{L}_{\text{CIB}}(\theta) = \min_{\theta}  I(X; Z) - \mu I(Z; Y \mid X)$,
where $\mu \ge 0$ controls the rate-distortion trade-off. To derive our final reward function, we rewrite the previous objective as a maximization problem, rather than a minimization one. Therefore, our objective becomes:
\begin{align}
    \max_{\theta} \mathcal{L}_{\text{CIB}}(\theta) = \max_{\theta} \underbrace{I(Z; Y \mid X)}_{\text{Sufficiency}} - \beta  \underbrace{I(X; Z)}_{\text{Minimality}} 
    \label{eq:cib_objective_2}
\end{align}
where $\beta$ gives direct control on the trade-off between accuracy and compression level (Figure~\ref{fig:side_side_1}, right). We refer the reader to Appendix~\ref{app:CIB} for the detailed discussion on the derivation. 
\paragraph{Deriving the Sufficiency Term (Accuracy Reward).}
We aim to maximize the conditional \ac{MI} $I(Z; Y \mid X)$. We can write it as a function of the policy $\pi_\theta(y|x,z), \pi_\theta(z|x)$ as: 
  \begin{align*}
       I(Y;Z|X)  & = \sum_{x,y,z} P(x,y) P(z|x,y)\log\frac{\pi_\theta(y|x,z)}{P(y|x)} \\ & = \sum_{x,y,z} P(x,y) \pi_\theta(z|x) \frac{\pi_\theta(y|x,z)}{P(y|x)} \log\frac{\pi_\theta(y|x,z)}{P(y|x)} \geq \sum_{x,y,z} P(x,y) \pi_\theta(z|x) \log\frac{\pi_\theta(y|x,z)}{P(y|x)}, 
  \end{align*}
  where we used the inequality $x\log x \geq \log x$ in the last step. 
  Note that the mutual information $I(Z; Y \mid X)$ can be decomposed as  $H(Y \mid X) - H(Y \mid X, Z)$. The first term $H(Y \mid X)$ represents the inherent difficulty of the dataset and is constant with respect to $\theta$. Thus, maximizing sufficiency is equivalent to minimizing the conditional entropy $H(Y \mid X, Z)$.
  We can  maximize the lower bound on it and approximate it further using the query-answer samples $(x_i,y_i)$. The first term of the optimization problem can then  be approximated as: $\sum_{i=1}^m \E_{Z\sim \pi_\theta(\cdot|x_i)} [\log \pi_\theta(y_i|x_i,Z)]$,
  where $m$ is the number of samples.
In many cases, like RLVR, a verifier $Q_\rho(y_i|x_i,z)$ is used to score the answer. Therefore, we can also optimize the following variational lower bound: $\sum_{i=1}^m \E_{Z\sim \pi_\theta(\cdot|x_i)} [\log Q_\rho(y_i|x_i,Z)]$.
See Appendix \ref{app:CIB} for the details of our derivation. In our experiments, we choose $Q_\rho(Y|X,Z)$ such that it gives a reward of 1 (0) for correct (wrong) answers.

\paragraph{From Log-Verifier to Binary Accuracy Reward.}
Our variational surrogate for sufficiency uses a verifier score $\log Q_\rho(y_i\mid x_i, z_i)$.
In our setting the verifier is deterministic, returning
$Q_\rho(y\mid x,z)\in\{0,1\}$ (1 if the extracted answer is correct, else 0),
so the log-score is ill-defined for incorrect answers.
We therefore use the $\varepsilon$-smoothed verifier $\widetilde Q_\rho(y\mid x,z)
:= \varepsilon + (1-\varepsilon)\,\mathbbm{1}\!\left(\widehat y(x,z)=y\right)$,
where $\varepsilon\in(0,1)$ and $\widehat y$ is the predicted answer. Then $\log \widetilde Q_\rho(y\mid x,z)
= \log \varepsilon - \log \varepsilon\, \mathbbm{1}\!\left(\widehat y(x,z)=y\right)$.
Since $\log\varepsilon$ is a constant w.r.t.\ $(\theta,z)$ and $-\log\varepsilon>0$,
maximizing $\mathbb{E}[\log \widetilde Q_\rho(y\mid x,z)]$
is \emph{equivalent} (up to an affine transformation) to maximizing
$\mathbb{E}[\mathbbm{1}(\widehat y(x,z)=y)]$.
Accordingly, we define the accuracy reward as: $r_{\mathrm{acc}}(x,y,z) := \mathbbm{1}\!\left(\widehat y(x,z)=y\right)
$,
which is a finite, stable surrogate for the log-verifier objective.
\paragraph{Deriving the Minimality Term (Information Cost).}
We aim to minimize the \ac{MI} $I(X; Z)$ to penalize redundancy in the \ac{CoT}: $I(X; Z) = \mathbb{E}_{X, Z} \left[ \log \frac{\pi_\theta(Z \mid X)}{P(Z)} \right]$.\\
However, computing $P(Z)$ is not tractable. Therefore, we introduce an unconditional variational prior $Q_\phi( Z)$ (a distribution over $Z$ that does not observe $X$) to find a variational bound similar to \cite{Alemi2017-variationalIB}. 
\begin{align}
\begin{aligned}
    I(X; Z) = &\ \mathbb{E}_{X, Z} \left[ \log \frac{\pi_\theta(Z \mid X)}{P(Z)} \right] \ 
    =\ \mathbb{E}_{X, Z} \left[ \log \frac{\pi_\theta(Z \mid X) Q_\phi(Z)}{P(Z) Q_\phi(Z)} \right] \\ 
    =&\ \mathbb{E}_{X, Z} \left[ \log \frac{\pi_\theta(Z \mid X)}{Q_\phi(Z)} \right] - \underbrace{D_{KL}(P(Z) \parallel Q_\phi(Z))}_{\ge 0}
\end{aligned}
\end{align}
Dropping the non-negative KL term gives the upper bound: $I(X; Z) \le \mathbb{E}_{X,Z} \left[ -\log Q_\phi(Z) \right] - H(Z \mid X)$,
where $Z \sim \pi_\theta(\cdot|X)$. 
We instantiate $Q_\phi(Z)$ using a frozen, pre-trained base model (not an instruction-finetuned model), ensuring it captures the statistics of general language without task-specific conditioning. More details on the derivation can be found in Appendix~\ref{app:CIB}.\\
The first term, $\mathbb{E}_{X,Z}[-\log Q_\phi(Z)]$, represents the cross-entropy rate (or description length) of the chain under the prior. It corresponds to the expected value of the reasoning trace information cost: $C(Z) := \sum_{t=1}^{|Z|} -\log Q_\phi(z_t \mid z_{<t})$. The second term, $-H(Z \mid X)$, corresponds to the negative entropy of the policy. In \ac{RL} algorithms like PPO, this term is naturally handled via an entropy regularization bonus to encourage exploration.

\paragraph{Reward Modeling.}
Combining the bounds, we aim to maximize the following objective:
\begin{align}
    \mathcal{L}_{\text{CIB}} = \mathbb{E}_{(X,Y)\sim P_\mathcal{D},Z \sim \pi_\theta}  \Big[
    \log \widetilde Q_\rho(Y|X,Z) + \beta \sum_{t=1}^{T} &\log Q_\phi(z_t \mid z_{<t}) \Big],
\end{align}
where the first term represents the accuracy score from the verifier, $\widetilde Q_\rho$, and $Q_\phi$ is chosen as prior distribution. 
We define our reward model as: $R(X,Y,Z)\;\coloneqq\;\ r_{\mathrm{acc}}(X,Y,Z) + \beta r_{\text{min}}(X,Z)$, where $r_{\mathrm{acc}}(X,Y,Z) := \mathbbm{1}\!\left(\widehat Y(X,Z)=Y\right)$ is the accuracy reward, taking a value of 1 if the predicted answer matches the ground truth $Y$, and 0 otherwise, and $r_{\text{min}}(X,Z) := \sum_{t=1}^{T} \log Q_\phi(z_t \mid z_{<t})$,is the cumulative surprisal (information cost) of the reasoning chain relative to the prior. 
Accuracy remains the primary objective, while $r_{\text{min}}$ acts as a semantic regularizer controlled by the coefficient $\beta$. 
This effectively assigns a semantic cost to every token: the cost $-\log Q_\phi$ penalizes low-probability (high-surprisal) tokens unless they contribute significantly to solving the task ($r_{\text{acc}}$). Tokens that are redundant or verbose increase the cumulative cost without improving accuracy, and are thus suppressed by the policy. 
\section{Theoretical Analysis: A Unified Framework}
\label{sec:theory}

A central motivation for this work is demonstrating that the \ac{CIB} serves as a general framework from which, e.g., length-based penalties naturally arise as a special case. As an example, we prove that length-constrained methods correspond to the \ac{CIB} rate term with non-informative priors. 


\begin{proposition}\label{prop:linear_penalty}
A standard length-based penalty (e.g., $g(Z) = \alpha f(|Z|)$) is equivalent to the \ac{CIB} objective under the assumption of a maximum entropy (uniform) prior, $Q$, over the vocabulary. 
\end{proposition}
\begin{proof}
Let $|V|$ be the vocabulary size and consider the minimality term $\sum -\log Q(z_t)$. A Maximum Entropy prior implies a uniform distribution over the vocabulary $V$ (i.e., $Q(z_t) = \frac{1}{|V|}$ for all $z_t$). Thus, the surprisal of every token becomes constant: $c = \log |V|$. Then, the total information cost for a \ac{CoT}, $Z$, of length $T$ becomes: $-\mathrm{log}Q(Z)=-\sum_{t=1}^{T}\mathrm{log}(\frac{1}{|V|}) = T\cdot \mathrm{log}|V|$.\\
Substituting this into the \ac{CIB} objective, the penalty term becomes $\beta T \mathrm{log}|V|$. By setting $\alpha = \beta \mathrm{log}|V|$, we recover a linear length penalty. This proves that linear penalties implicitly assume that all tokens carry equal information content ($\mathrm{log}|V|$), ignoring the underlying semantics of the \ac{CoT}. We provide further experimental analysis of uniform prior and its relation to the semantic prior in Appendix~\ref{app:uniform_prior}.
\end{proof}
\begin{proposition}\label{prop:lcpo}
Target-length penalties, such as LCPO-Exact~\citep{aggarwal2025l1}, correspond to the \ac{CIB} objective with a Laplace prior.
\end{proposition}
\begin{proof} 
Any penalty function $g(Z)$ applied to the reward can be interpreted as an implicit prior $Q(Z) \propto \exp(-g(Z))$. 
LCPO-Exact penalizes deviation from a target length $n_{gold}$ via the term $g(Z) = |n_{gold} - n_y|$, where $n_y$ is the length of the generated \ac{CoT}. The corresponding implicit prior is: $Q_{\text{LCPO}}(Z) \propto e^{-|n_{gold} - n_y|}$. This is a Laplace-like distribution over the sequence length $T$, centered at $n_{gold}$. Interpreting LCPO through this lens reveals a strong inductive bias: it posits that there exists a golden length for reasoning length, and any deviation (shorter or longer) is exponentially improbable. 
\end{proof}
Crucially, in both Propositions~\ref{prop:linear_penalty} and~\ref{prop:lcpo}, the implicit prior $Q(Z)$ depends solely on the sequence length, whereas our proposed \ac{CIB} method uses a language model prior defining a per-token cost.

\section{Experimental Results} \label{sec:exp_res}

\paragraph{Training and Evaluation Details}
We evaluate \ac{CIB} on three families of small reasoning models: DLER-\{1.5B, 7B\}~\citep{liu2025dler}, OpenReasoning-Nemotron-1.5B~\citep{ahmad2025opencodereasoning,moshkov2025aimo}, and DeepScaleR-1.5B~\citep{luo2025deepscalerdataset}. We focus on the 1.5B--7B range, as it is the primary deployment target for efficient reasoning under latency and memory constraints.
Training uses GRPO~\citep{shao2024deepseekmath} with a group size of 16 and a Qwen2.5-Base-\{1.5B, 7B\} prior; the prior is used at training time only and adds no inference cost.
Our main evaluation covers five mathematical reasoning benchmarks---MATH500~\citep{lightman2023lets}, AIME24~\citep{aime2024}, AIME25~\citep{aime2025}, Minerva~\citep{lewkowycz2022solving}, and OlympiadBench~\citep{he2024olympiadbench}---which provide unambiguous, deterministic targets and therefore a controlled evaluation setting. To assess whether \ac{CIB} generalizes beyond math, we also evaluate on GPQA~\citep{rein2023gpqa}, MMLU-Pro~\citep{wang2024mmlu}, and LiveCodeBench~\citep{jain2024livecodebench}. We follow the same protocol as in~\citet{liu2025dler}, using vLLM as the inference engine (temperature 0.6, $\mathrm{top}_p = 0.95$, max tokens 32K, 16 generations per prompt), and report pass@1 accuracy. The standard error of the mean (SEM) is computed over 10 seeds for every benchmark except MATH500 and LiveCodeBench. Further training and evaluation details are in Appendix~\ref{app:tr_details}, and the list of assets used in the paper in Appendix~\ref{app:assets}.

\subsection{Main Results}
We first verify that the minimality reward (Section~\ref{sec:methodology}) provides a usable learning signal. Figure~\ref{fig:side_side_1} (left) shows a negative correlation with completion length, with dispersion suggesting sensitivity to the token sequence beyond length alone.\\
Table~\ref{tab:results} reports accuracy and average CoT length on the mathematical benchmarks. For each model family we evaluate \ac{CIB} at two compression strengths ($\beta^{-}\!=\!5.0\!\times\!10^{-5}$, $\beta^{+}\!=\!1.5\!\times\!10^{-4}$), selected on the training distribution to target mild and more aggressive compression, and with two prior sizes (1.5B and 7B). For DeepScaleR-1.5B, we include external baselines where available: L3L1-EXACT~\citep{aggarwal2025l1}, evaluated from the authors' public checkpoint, and ThinkPrune-I2k~\citep{hou2025thinkprune}, HAPO~\citep{huang2025hapo}, and LAPO-I~\citep{wu2025lapo} (marked $^{\S}$), with results quoted directly from~\citet{wu2025lapo}; some benchmark entries are therefore unavailable. Further comparisons are reported in Appendix~\ref{app:literature_results}.

\begin{table*}[!ht]
\centering
\caption{\textbf{Math benchmarks: accuracy ($\uparrow$) and average CoT length ($\downarrow$).} Bold indicates best performance within each model group. \textbf{\textcolor{DarkRed}{Red}} marks significant degradation, \textcolor{Green}{green} highlights improvement over the baseline. CIB subscripts denote the compression setting ($\beta^-\!=\!5 \times 10^{-5}$, $\beta^+\!=\!1.5 \times 10^{-4}$) and prior size (1.5B or 7B). No SEM available for baselines reported from literature.\\
$^{*}$: Pass@1 with $n=16$ for consistency, the official HuggingFace page reports pass@1(Avg@64);\\
$^\dagger$: Our models; $^\ddagger$: Models we trained using the L1-EXACT length penalty~\citep{aggarwal2025l1}; $^{\S}$: Results from~\cite{wu2025lapo}.}
\label{tab:results}
\resizebox{\textwidth}{!}{
\begin{tabular}{lcccccc}
\toprule
Model & MATH500 & AIME24 & AIME25 & Minerva & Olympiad & Avg. \\
 & Acc.\,/\,Tok. & Acc.\,/\,Tok. & Acc.\,/\,Tok. & Acc.\,/\,Tok. & Acc.\,/\,Tok. & Acc.\,/\,Tok. \\
\midrule
\multicolumn{7}{@{}l}{\textit{DeepScaleR-1.5B}} \\
\cmidrule{1-7}
Baseline & $85.8_{\pm1.6}$\,/\,$3190_{\pm132}$ & $38.2_{\pm1.5}$\,/\,$9331_{\pm116}$ & $25.4_{\pm1.5}$\,/\,$8907_{\pm144}$ & $19.9_{\pm0.8}$\,/\,$6386_{\pm95}$ & $54.8_{\pm0.4}$\,/\,$5922_{\pm29}$ & $44.8$\,/\,$6747$ \\
L3L1-EXACT & $85.2_{\pm1.6}$\,/\,$1932_{\pm60}$ & $\mathbf{\textcolor{DarkRed}{24.4}}_{\pm1.3}$\,/\,$\mathbf{2637}_{\pm15}$ & $\mathbf{\textcolor{DarkRed}{20.2}}_{\pm1.3}$\,/\,$\mathbf{2608}_{\pm17}$ & $19.8_{\pm0.8}$\,/\,$2124_{\pm29}$ & $\mathbf{\textcolor{DarkRed}{48.8}}_{\pm0.3}$\,/\,$\mathbf{2369}_{\pm6}$ & $\mathbf{\textcolor{DarkRed}{39.7}}$\,/\,$\mathbf{\textcolor{Green}{2334}}$ \\
ThinkPrune-I2k$^{\S}$ & $85.5_{\pm}$\,/\,$1707_{\pm}$ & $34.9_{\pm}$\,/\,$5095_{\pm}$ & $-$\,/\,$-$ & $-$\,/\,$-$ & $54.7_{\pm}$\,/\,$3498_{\pm}$ & $-$\,/\,$-$ \\
HAPO$^{\S}$ & $84.4_{\pm}$\,/\,$2370_{\pm}$ & $31.4_{\pm}$\,/\,$7702_{\pm}$ & $-$\,/\,$-$ & $-$\,/\,$-$ & $51.4_{\pm}$\,/\,$4571_{\pm}$ & $-$\,/\,$-$ \\
LAPO-I$^{\S}$ & $86.3_{\pm}$\,/\,$2168_{\pm}$ & $38.1_{\pm}$\,/\,$5371_{\pm}$ & $-$\,/\,$-$ & $-$\,/\,$-$ & $\mathbf{56.3}_{\pm}$\,/\,$4024_{\pm}$ & $-$\,/\,$-$ \\
\rowcolor{cibgray}
CIB$^\dagger_{\beta^-\!,\,1.5B}$ & $\mathbf{87.6}_{\pm1.5}$\,/\,$2359_{\pm121}$ & $\mathbf{40.2}_{\pm1.9}$\,/\,$7298_{\pm141}$ & $\mathbf{27.8}_{\pm2.1}$\,/\,$6866_{\pm170}$ & $\mathbf{20.4}_{\pm0.8}$\,/\,$4510_{\pm77}$ & $55.5_{\pm0.7}$\,/\,$4472_{\pm47}$ & $\mathbf{\textcolor{Green}{46.3}}$\,/\,$5101$ \\
\rowcolor{cibgray}
CIB$^\dagger_{\beta^+\!,\,1.5B}$ & $87.2_{\pm1.5}$\,/\,$2017_{\pm107}$ & $38.2_{\pm2.0}$\,/\,$6937_{\pm141}$ & $27.6_{\pm2.1}$\,/\,$6757_{\pm184}$ & $20.1_{\pm0.8}$\,/\,$4147_{\pm73}$ & $55.3_{\pm0.7}$\,/\,$4211_{\pm45}$ & $45.7$\,/\,$4814$ \\
\rowcolor{cibgray}
CIB$^\dagger_{7B}$ & $86.1_{\pm1.6}$\,/\,$\mathbf{1337}_{\pm93}$ & $37.5_{\pm1.9}$\,/\,$5003_{\pm143}$ & $25.0_{\pm1.9}$\,/\,$5389_{\pm164}$ & $19.7_{\pm0.8}$\,/\,$\mathbf{2058}_{\pm68}$ & $55.0_{\pm0.7}$\,/\,$3065_{\pm41}$ & $44.7$\,/\,$\mathbf{\textcolor{Green}{3370}}$ \\
\midrule
\multicolumn{7}{@{}l}{\textit{Nemotron-1.5B}} \\
\cmidrule{1-7}
Baseline$^{*}$ & $85.8_{\pm1.6}$\,/\,$8509_{\pm382}$ & $43.1_{\pm1.2}$\,/\,$22931_{\pm326}$ & $\mathbf{35.1}_{\pm1.3}$\,/\,$24086_{\pm333}$ & $11.0_{\pm0.5}$\,/\,$13011_{\pm191}$ & $43.8_{\pm0.3}$\,/\,$16270_{\pm89}$ & $43.8$\,/\,$16961$ \\
L1-Exact$^\ddagger$ & $85.8_{\pm1.5}$\,/\,$8807_{\pm393}$ & $40.2_{\pm1.4}$\,/\,$22079_{\pm383}$ & $30.8_{\pm1.4}$\,/\,$23243_{\pm400}$ & $11.4_{\pm0.4}$\,/\,$12925_{\pm147}$ & $44.0_{\pm0.4}$\,/\,$15316_{\pm127}$ & $42.4$\,/\,$16474$ \\
\rowcolor{cibgray}
CIB$^\dagger_{\beta^-\!,\,1.5B}$ & $\mathbf{86.8}_{\pm1.6}$\,/\,$7536_{\pm412}$ & $\mathbf{45.3}_{\pm1.9}$\,/\,$21962_{\pm483}$ & $34.1_{\pm1.9}$\,/\,$22980_{\pm547}$ & $\mathbf{11.7}_{\pm0.6}$\,/\,$11118_{\pm189}$ & $\mathbf{52.7}_{\pm0.6}$\,/\,$15089_{\pm176}$ & $\mathbf{\textcolor{Green}{46.1}}$\,/\,$15737$ \\
\rowcolor{cibgray}
CIB$^\dagger_{\beta^+\!,\,1.5B}$ & $85.4_{\pm1.6}$\,/\,$5982_{\pm354}$ & $41.5_{\pm1.9}$\,/\,$20954_{\pm454}$ & $34.0_{\pm1.8}$\,/\,$21407_{\pm542}$ & $10.9_{\pm0.6}$\,/\,$10051_{\pm168}$ & $50.8_{\pm0.6}$\,/\,$13894_{\pm173}$ & $44.5$\,/\,$14458$ \\
\rowcolor{cibgray}
CIB$^\dagger_{7B}$ & $86.7_{\pm1.5}$\,/\,$\mathbf{5789}_{\pm366}$ & $42.7_{\pm1.8}$\,/\,$\mathbf{20547}_{\pm470}$ & $32.3_{\pm1.8}$\,/\,$\mathbf{20678}_{\pm532}$ & $10.6_{\pm0.5}$\,/\,$\mathbf{8752}_{\pm177}$ & $51.1_{\pm0.6}$\,/\,$\mathbf{12242}_{\pm170}$ & $44.7$\,/\,$\mathbf{\textcolor{Green}{13602}}$ \\
\midrule
\multicolumn{7}{@{}l}{\textit{DLER-1.5B}} \\
\cmidrule{1-7}
Baseline & $86.4_{\pm1.5}$\,/\,$1873_{\pm91}$ & $33.4_{\pm1.4}$\,/\,$3478_{\pm25}$ & $23.9_{\pm1.5}$\,/\,$3290_{\pm29}$ & $20.9_{\pm0.8}$\,/\,$2587_{\pm29}$ & $54.1_{\pm0.4}$\,/\,$2766_{\pm10}$ & $43.7$\,/\,$2799$ \\
L1-Exact$^\ddagger$ & $86.6_{\pm1.5}$\,/\,$1751_{\pm111}$ & $35.0_{\pm2.1}$\,/\,$3274_{\pm34}$ & $22.7_{\pm2.0}$\,/\,$3037_{\pm41}$ & $21.2_{\pm0.8}$\,/\,$2277_{\pm26}$ & $54.0_{\pm0.7}$\,/\,$2577_{\pm20}$ & $43.9$\,/\,$2583$ \\
\rowcolor{cibgray}
CIB$^\dagger_{\beta^-\!,\,1.5B}$ & $\mathbf{88.4}_{\pm1.5}$\,/\,$1687_{\pm102}$ & $\mathbf{35.4}_{\pm2.0}$\,/\,$3198_{\pm40}$ & $\mathbf{25.3}_{\pm2.1}$\,/\,$3260_{\pm56}$ & $\mathbf{21.6}_{\pm0.8}$\,/\,$2346_{\pm27}$ & $\mathbf{54.3}_{\pm0.7}$\,/\,$2616_{\pm18}$ & $\mathbf{\textcolor{Green}{45.0}}$\,/\,$2621$ \\
\rowcolor{cibgray}
CIB$^\dagger_{\beta^+\!,\,1.5B}$ & $87.6_{\pm1.4}$\,/\,$1421_{\pm64}$ & $34.1_{\pm2.0}$\,/\,$3124_{\pm43}$ & $24.2_{\pm2.1}$\,/\,$3027_{\pm47}$ & $21.0_{\pm0.8}$\,/\,$2305_{\pm27}$ & $54.1_{\pm0.7}$\,/\,$2544_{\pm17}$ & $44.2$\,/\,$2484$ \\
\rowcolor{cibgray}
CIB$^\dagger_{7B}$ & $85.0_{\pm1.6}$\,/\,$\mathbf{805}_{\pm35}$ & $31.6_{\pm2.0}$\,/\,$\mathbf{2504}_{\pm37}$ & $22.4_{\pm1.9}$\,/\,$\mathbf{2257}_{\pm38}$ & $20.4_{\pm0.9}$\,/\,$\mathbf{1327}_{\pm23}$ & $51.6_{\pm0.7}$\,/\,$\mathbf{1631}_{\pm16}$ & $42.2$\,/\,$\mathbf{\textcolor{Green}{1705}}$ \\
\midrule
\multicolumn{7}{@{}l}{\textit{DLER-7B}} \\
\cmidrule{1-7}
Baseline & $93.6_{\pm1.1}$\,/\,$1159_{\pm41}$ & $49.6_{\pm1.7}$\,/\,$3221_{\pm47}$ & $36.6_{\pm1.6}$\,/\,$3271_{\pm54}$ & $27.3_{\pm1.0}$\,/\,$1849_{\pm30}$ & $\mathbf{65.0}_{\pm0.4}$\,/\,$2840_{\pm27}$ & $54.4$\,/\,$2468$ \\
L1-Exact$^\ddagger$ & $92.8_{\pm1.2}$\,/\,$675_{\pm33}$ & $46.2_{\pm2.5}$\,/\,$2982_{\pm62}$ & $31.2_{\pm2.2}$\,/\,$2564_{\pm61}$ & $26.1_{\pm1.0}$\,/\,$1144_{\pm30}$ & $61.9_{\pm0.7}$\,/\,$1872_{\pm26}$ & $51.6$\,/\,$1847$ \\
\rowcolor{cibgray}
CIB$^\dagger_{\beta^-,7B}$ & $\mathbf{94.0}_{\pm1.1}$\,/\,$1033_{\pm39}$ & $\mathbf{49.7}_{\pm2.3}$\,/\,$2642_{\pm69}$ & $\mathbf{37.1}_{\pm2.2}$\,/\,$2962_{\pm77}$ & $\mathbf{27.6}_{\pm1.0}$\,/\,$1381_{\pm26}$ & $64.1_{\pm0.7}$\,/\,$2361_{\pm39}$ & $\mathbf{\textcolor{Green}{54.5}}$\,/\,$2076$ \\
\rowcolor{cibgray}
CIB$^\dagger_{\beta^+,7B}$ & $92.2_{\pm1.2}$\,/\,$\mathbf{678}_{\pm34}$ & $48.2_{\pm2.2}$\,/\,$\mathbf{2419}_{\pm75}$ & $35.4_{\pm2.3}$\,/\,$\mathbf{2550}_{\pm72}$ & $25.8_{\pm1.0}$\,/\,$\mathbf{851}_{\pm19}$ & $62.7_{\pm0.7}$\,/\,$\mathbf{1704}_{\pm28}$ & $52.9$\,/\,$\mathbf{\textcolor{Green}{1640}}$ \\
\bottomrule
\end{tabular}
}
\vspace{-10pt}
\end{table*}

\begin{table*}[!ht]
\centering
\caption{\textbf{Science, Knowledge \& Coding benchmarks: Accuracy ($\uparrow$) and average CoT length ($\downarrow$).} Bold indicates best performance within each model group. \textbf{\textcolor{DarkRed}{Red}} marks significant degradation, \textbf{\textcolor{Green}{green}} highlights improvement over the baseline. CIB subscripts denote the compression setting ($\beta^-\!=\!5 \times 10^{-5}$, $\beta^+\!=\!1.5 \times 10^{-4}$) and prior size (1.5B or 7B). No SEM available for baselines reported from literature.\\
$^{*}$: Pass@1 with $n=16$ for consistency, the official HuggingFace page reports pass@1(Avg@64);\\
$^\dagger$: Our models; $^\ddagger$: Models we trained using the L1-EXACT length penalty~\citep{aggarwal2025l1}.}
\label{tab:res_beyond_math}
\footnotesize
\begin{tabular}{lcccc}
\toprule
Model & GPQA Diamond & MMLU-Pro & LCB & Avg. \\
 & Acc.\,/\,Tok. & Acc.\,/\,Tok. & Acc.\,/\,Tok. & Acc.\,/\,Tok. \\
\midrule
\multicolumn{5}{@{}l}{\textit{DeepScaleR-1.5B}} \\
\cmidrule{1-5}
Baseline & $36.5_{\pm0.7}$\,/\,$7445_{\pm47}$ & $39.4_{\pm0.2}$\,/\,$5285_{\pm16}$ & $18.7_{\pm2.6}$\,/\,$11403_{\pm301}$ & $31.5$\,/\,$8044$ \\
L3L1-EXACT & $37.9_{\pm0.7}$\,/\,$3125_{\pm16}$ & $40.4_{\pm0.2}$\,/\,$2240_{\pm5}$ & $17.5_{\pm2.5}$\,/\,$\mathbf{4612}_{\pm89}$ & $31.9$\,/\,$\mathbf{\textcolor{Green}{3326}}$ \\
\rowcolor{cibgray}
CIB$^\dagger_{\beta^-\!,\,1.5B}$ & $\mathbf{38.4}_{\pm0.6}$\,/\,$4571_{\pm98}$ & $\mathbf{41.6}_{\pm0.2}$\,/\,$3286_{\pm14}$ & $\mathbf{18.8}_{\pm1.2}$\,/\,$10410_{\pm151}$ & $\mathbf{\textcolor{Green}{32.9}}$\,/\,$6089$ \\
\rowcolor{cibgray}
CIB$^\dagger_{\beta^+\!,\,1.5B}$ & $35.5_{\pm0.6}$\,/\,$4429_{\pm87}$ & $40.7_{\pm0.2}$\,/\,$3108_{\pm14}$ & $18.6_{\pm1.1}$\,/\,$10274_{\pm134}$ & $31.6$\,/\,$5937$ \\
\rowcolor{cibgray}
CIB$^\dagger_{7B}$ & $39.5_{\pm0.6}$\,/\,$\mathbf{2850}_{\pm87}$ & $40.8_{\pm0.2}$\,/\,$\mathbf{2118}_{\pm13}$ & $17.4_{\pm1.1}$\,/\,$9481_{\pm126}$ & $32.6$\,/\,$4816$ \\
\midrule
\multicolumn{5}{@{}l}{\textit{Nemotron-1.5B}} \\
\cmidrule{1-5}
Baseline$^{*}$ & $31.4_{\pm0.7}$\,/\,$21238_{\pm154}$ & $41.4_{\pm0.2}$\,/\,$10490_{\pm39}$ & $26.6_{\pm2.5}$\,/\,$24082_{\pm620}$ & $33.1$\,/\,$18603$ \\
L1-Exact$^\ddagger$ & $33.9_{\pm1.1}$\,/\,$21132_{\pm219}$ & $39.6_{\pm0.1}$\,/\,$10416_{\pm30}$ & $\mathbf{26.8}_{\pm1.3}$\,/\,$23995_{\pm310}$ & $33.4$\,/\,$18514$ \\
\rowcolor{cibgray}
CIB$^\dagger_{\beta^-\!,\,1.5B}$ & $\mathbf{34.9}_{\pm0.6}$\,/\,$22558_{\pm299}$ & $39.4_{\pm0.2}$\,/\,$9851_{\pm41}$ & $26.1_{\pm1.2}$\,/\,$23557_{\pm295}$ & $33.5$\,/\,$18655$ \\
\rowcolor{cibgray}
CIB$^\dagger_{\beta^+\!,\,1.5B}$ & $34.8_{\pm0.6}$\,/\,$19694_{\pm307}$ & $39.6_{\pm0.2}$\,/\,$8908_{\pm38}$ & $26.1_{\pm1.2}$\,/\,$22227_{\pm308}$ & $33.5$\,/\,$16943$ \\
\rowcolor{cibgray}
CIB$^\dagger_{7B}$ & $33.6_{\pm0.7}$\,/\,$\mathbf{19533}_{\pm319}$ & $\mathbf{41.6}_{\pm0.2}$\,/\,$\mathbf{8496}_{\pm39}$ & $25.9_{\pm1.2}$\,/\,$\mathbf{21817}_{\pm304}$ & $\mathbf{\textcolor{Green}{33.7}}$\,/\,$\mathbf{\textcolor{Green}{16615}}$ \\
\midrule
\multicolumn{5}{@{}l}{\textit{DLER-1.5B}} \\
\cmidrule{1-5}
Baseline & $\mathbf{40.9}_{\pm0.7}$\,/\,$3280_{\pm16}$ & $41.7_{\pm0.2}$\,/\,$2575_{\pm12}$ & $17.9_{\pm2.5}$\,/\,$4857_{\pm89}$ & $\mathbf{\textcolor{Green}{33.5}}$\,/\,$3571$ \\
L1-Exact$^\ddagger$ & $36.8_{\pm0.6}$\,/\,$3012_{\pm30}$ & $40.8_{\pm0.2}$\,/\,$2326_{\pm11}$ & $17.8_{\pm1.1}$\,/\,$4585_{\pm39}$ & $31.8$\,/\,$3308$ \\
\rowcolor{cibgray}
CIB$^\dagger_{\beta^-\!,\,1.5B}$ & $39.4_{\pm0.6}$\,/\,$3183_{\pm47}$ & $\mathbf{41.8}_{\pm0.2}$\,/\,$2469_{\pm13}$ & $\mathbf{18.2}_{\pm1.3}$\,/\,$4727_{\pm46}$ & $33.1$\,/\,$3460$ \\
\rowcolor{cibgray}
CIB$^\dagger_{\beta^+\!,\,1.5B}$ & $38.8_{\pm0.5}$\,/\,$3044_{\pm36}$ & $41.6_{\pm0.2}$\,/\,$2353_{\pm12}$ & $17.9_{\pm1.1}$\,/\,$4629_{\pm42}$ & $32.8$\,/\,$3342$ \\
\rowcolor{cibgray}
CIB$^\dagger_{7B}$ & $36.7_{\pm0.6}$\,/\,$\mathbf{2423}_{\pm43}$ & $40.4_{\pm0.2}$\,/\,$\mathbf{1973}_{\pm15}$ & $16.8_{\pm1.1}$\,/\,$\mathbf{3977}_{\pm40}$ & $31.3$\,/\,$\mathbf{\textcolor{Green}{2791}}$ \\
\midrule
\multicolumn{5}{@{}l}{\textit{DLER-7B}} \\
\cmidrule{1-5}
Baseline & $50.1_{\pm0.8}$\,/\,$3255_{\pm19}$ & $\mathbf{58.7}_{\pm0.2}$\,/\,$2012_{\pm5}$ & $\mathbf{35.9}_{\pm3.2}$\,/\,$4847_{\pm130}$ & $48.2$\,/\,$3371$ \\
L1-Exact$^\ddagger$ & $48.6_{\pm0.7}$\,/\,$2551_{\pm40}$ & $56.9_{\pm0.2}$\,/\,$1446_{\pm6}$ & $34.4_{\pm1.4}$\,/\,$4120_{\pm52}$ & $46.6$\,/\,$2706$ \\
\rowcolor{cibgray}
CIB$^\dagger_{\beta^-,7B}$ & $\mathbf{50.5}_{\pm0.8}$\,/\,$2899_{\pm33}$ & $\mathbf{58.7}_{\pm0.2}$\,/\,$1718_{\pm5}$ & $35.6_{\pm1.4}$\,/\,$4003_{\pm55}$ & $\mathbf{\textcolor{Green}{48.3}}$\,/\,$2873$ \\
\rowcolor{cibgray}
CIB$^\dagger_{\beta^+,7B}$ & $49.7_{\pm0.7}$\,/\,$\mathbf{2476}_{\pm32}$ & $58.2_{\pm0.2}$\,/\,$\mathbf{1430}_{\pm6}$ & $34.5_{\pm1.4}$\,/\,$\mathbf{3796}_{\pm61}$ & $47.5$\,/\,$\mathbf{\textcolor{Green}{2567}}$ \\
\bottomrule
\end{tabular}
\vspace{-10pt}
\end{table*}

\paragraph{CIB Improves Accuracy while Compressing Reasoning.}
Across all four model groups, CIB$_{\beta^-}$ improves average accuracy over the uncompressed baseline while simultaneously shortening the CoT (Table~\ref{tab:results}).
DeepScaleR-1.5B offers the broadest comparison against length-based penalties: CIB$_{\beta^-}$ gains $1.5$ points of average accuracy and cuts average CoT length by ${\sim}24\%$, with improvements consistent across both easy and hard benchmarks. 
Length-based penalties, in contrast, often trade accuracy for shorter outputs: L3L1-EXACT loses $13.8$ points on AIME24 despite strong compression.
CIB$_{7B}$ is both more accurate and more compressive than ThinkPrune-I2k on every shared benchmark, while CIB$_{\beta^-,1.5B}$ is more accurate still, at the cost of slightly longer traces. LAPO-I sits between the two on the accuracy--compression plane: CIB$_{\beta^-,1.5B}$ is more accurate on MATH500 and AIME24, CIB$_{7B}$ compresses more aggressively, and LAPO-I's apparent accuracy advantage over CIB$_{7B}$ is within uncertainty.

The two levers behave predictably. Raising $\beta$ from $\beta^{-}$ to $\beta^{+}$ trades up to ${\sim}1.6$ points of accuracy for $5\%$--$21\%$ fewer tokens (Table~\ref{tab:results}, Avg.\ column). The 7B prior enables deeper compression, yielding the shortest CoT in every group; the only shorter model is L3L1-EXACT on DeepScaleR-1.5B, whose compression comes with substantial accuracy degradation (Table~\ref{tab:results}, red entries).
These trends hold across families. On Nemotron-1.5B, CIB$_{\beta^-,1.5B}$ achieves the largest accuracy gain ($+2.3$ points) with ${\sim}7\%$ fewer tokens, while CIB$_{7B}$ removes ${\sim}20\%$ and still improves on the baseline. \ac{CIB} also scales to larger models: on DLER-7B, CIB$_{\beta^-,7B}$ matches the baseline's strong accuracy with a ${\sim}16\%$ shorter CoT, while CIB$_{\beta^+,7B}$ reduces length by roughly one-third at a cost of $1.5$ accuracy points.

\paragraph{Generalization Beyond Math.}
Although trained exclusively on math data, \ac{CIB} benefits transfer to science, knowledge, and code benchmarks (Table~\ref{tab:res_beyond_math}). On DeepScaleR-1.5B, CIB$_{\beta^-}$ improves GPQA Diamond by $1.9$ points ($36.5_{\pm0.7}\!\to\!38.4_{\pm0.6}$) and MMLU-Pro by $2.2$ points ($39.4_{\pm0.2}\!\to\!41.6_{\pm0.2}$) while matching LiveCodeBench ($18.7_{\pm2.6}\!\to\!18.8_{\pm1.2}$), with token reductions of up to $38.6\%$. On DLER-7B, it performs on par with the baseline across all three benchmarks 
while using $10.9\%$--$17.4\%$ fewer tokens. LiveCodeBench stays within $1$ point of the baseline for $\beta^-$ and $\beta^+$ in every family; only the most aggressive 7B-prior configurations exceed that margin (DeepScaleR $-1.3$; DLER-1.5B $-1.1$).
Two family-specific patterns are stand out. On DLER-1.5B, the baseline holds the best GPQA Diamond accuracy in its group ($40.9_{\pm0.7}$) and every compression method lowers it, but CIB$_{\beta^-}$ loses only $1.5$ points ($39.4_{\pm0.6}$) against L1-Exact's $4.1$ ($36.8_{\pm0.6}$), indicating that \ac{CIB} preserves more domain-specific knowledge even where some regression is unavoidable. On Nemotron, the 1.5B-prior variants cost up to $2.0$ points on MMLU-Pro ($41.4\!\to\!39.4$), which the 7B prior fully recovers ($41.6$, $+0.2$ over the baseline) while still compressing by $19.0\%$ (${\sim}10.5\text{K}\!\to\!{\sim}8.5\text{K}$), suggesting that a stronger prior compensates for the capacity gap on knowledge-intensive tasks.

\paragraph{Statistical Reliability.}
On competition-level benchmarks (AIME24, AIME25), individual CIB$_{\beta^-}$ gains are comparable in magnitude to the SEM (DeepScaleR AIME24: $38.2_{\pm1.5}\!\to\!40.2_{\pm1.9}$; DLER-1.5B AIME24: $33.4_{\pm1.4}\!\to\!35.4_{\pm2.0}$), as expected for 30-problem test sets. The key signal is \emph{consistency}: CIB$_{\beta^-}$ raises accuracy on 25 of 32 model$\times$benchmark evaluations across Tables~\ref{tab:results} and~\ref{tab:res_beyond_math}. The length-penalty degradations, by contrast, are individually unambiguous---L3L1-EXACT loses $13.8$ points on AIME24 ($38.2_{\pm1.5}\!\to\!24.4_{\pm1.3}$) and DLER-7B L1-Exact $3.1$ points on OlympiadBench ($65.0_{\pm0.4}\!\to\!61.9_{\pm0.7}$)---well beyond overlapping confidence intervals.

\paragraph{Controlling the Accuracy--Efficiency Trade--Off.}
The weight $\beta$ in the \ac{CIB} objective (Equation~\ref{eq:cib_objective_2}) gives fine-grained control over the accuracy--efficiency trade--off. Sweeping $\beta$ traces a smooth Pareto frontier (Figure~\ref{fig:side_side_1}, right): low values prioritize accuracy, high values compress aggressively. On AIME24, intermediate settings retain most of the baseline accuracy at a substantially lower token count, while extreme values reveal a clear knee beyond which compression degrades reasoning quality. This continuous knob lets practitioners pick the operating point matching their latency or cost constraints.

\paragraph{Qualitative CoT Comparison.}
Inspecting reasoning traces on arithmetic and symbolic tasks shows that \ac{CIB} systematically removes conversational scaffolding, redundant verification loops, and tautological checks. Unlike naive truncation, the semantic prior alters the reasoning topology, preserving the ``computational bridge'' while filtering transitions with low marginal information about $Y$. Case studies are in Appendix~\ref{app:qualitative_analysis} (Figures~\ref{fig:cot_1}--\ref{fig:cot_3}).

\paragraph{Prior Overhead.}
A natural concern is whether the prior becomes a training bottleneck. Since $Q_\phi$ is frozen, it needs no gradient accumulation, backward passes, or optimizer states: it adds only a static VRAM footprint and a single forward pass over the generated trajectory, scoring all tokens in parallel. 
On our distributed setup (4$\times$H100, 8 generations per sample, vLLM backend), Table~\ref{tab:prior_efficiency} compares wall-clock time across four prior sizes against a length-based penalty reward for DLER-1.5B.
Scaling the prior to 7B reduces rollout token throughput by less than $6\%$ and the effective data collection rate by less than $6.2\%$ ($1.47\!\to\!1.38$ samples/sec), so training is largely insensitive to prior size in wall-clock terms. Prior \emph{capacity}, however, does matter for the quality of the semantic signal. 
A prior that is too weak has a predictive distribution flattening toward uniform, and the semantic cost degenerates into a naive length penalty that cannot separate redundant tokens from critical reasoning steps. 
\begin{wraptable}{r}{0.65\columnwidth}
\vspace{-5pt}
\centering
\caption{\textbf{Semantic prior training overhead.} Wall-clock comparison for DLER-1.5B. 
${}^\dagger$: No prior model.
}
\label{tab:prior_efficiency}
\begin{tabular}{lccc}
\toprule
\textbf{Prior} & \textbf{Step Time (s)} & \textbf{Tokens/sec} & \textbf{Samples/sec} \\
\midrule
Length-based${}^\dagger$ & $21.1_{\pm0.3}$ & $3691_{\pm92}$ & $1.47_{\pm0.02}$ \\ $Q_\phi=$ 0.5B & $21.8_{\pm0.3}$ & $3677_{\pm97}$ & $1.46_{\pm0.03}$ \\ $Q_\phi=$ 1.5B & $22.5_{\pm0.3}$ & $3482_{\pm96}$ & $1.42_{\pm0.03}$ \\ $Q_\phi=$ 3B & $22.4_{\pm0.4}$ & $3489_{\pm99}$ & $1.43_{\pm0.03}$ \\ $Q_\phi=$ 7B & $23.2_{\pm0.4}$ & $3449_{\pm98}$ & $1.38_{\pm0.03}$ \\
\bottomrule
\end{tabular}
\vspace{0pt} 
\end{wraptable}

A cross-family prior 
introduces tokenization-dependent confounds instead: vocabulary and segmentation mismatches can penalize the policy's native formatting rather than its actual verbosity. Coarser-grained costs that mitigate this effect are an interesting direction for future works.

\subsection{Scaling Properties}\label{sec:scaling_prop}
Figure~\ref{fig:itc_grid} shows the inference-time compute (ITC) behaviour of our \ac{CIB} models against the baselines on AIME24 and AIME25, measured as pass@$k$ for $k$ up to 128.
At the standard 8K budget (first and last columns), \ac{CIB}-compressed models scale on par with or better than their base models, despite producing substantially shorter reasoning traces.
On DeepScaleR-1.5B, both priors improve over the base model at every $k$
and the advantage of \ac{CIB} grows as the budget tightens.
When the maximum generation length is bounded to 3K (second column) or 2K (third column), \ac{CIB} with the large prior ($Q_\phi\!=\!7B$) achieves the best scaling performance in almost every setting.
At 2K tokens it nearly doubles the pass@1 of DLER on AIME24 ($0.21$ vs.\ $0.11$) and keeps a clear margin up to pass@128.
Notably, under the 2K budget the small prior ($Q_\phi\!=\!1.5B$) falls slightly below the base model at large $k$, which indicates that a stronger prior is key to retaining useful reasoning when the budget is severely constrained.

\begin{figure}[!ht]
    \centering
    \includegraphics[width=\textwidth]{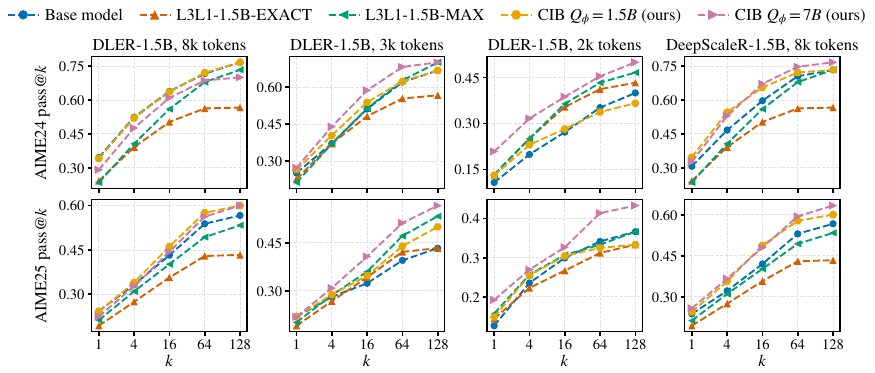}
    \caption{\textbf{Inference-time compute under different generation budgets.} Pass@$k$ on AIME24 (top) and AIME25 (bottom). Columns 1--3: DLER-1.5B base model with a maximum generation length of 8K, 3K, and 2K tokens, respectively; column 4: DeepScaleR-1.5B base models with 8K tokens.}
    \label{fig:itc_grid}
    \vspace{-5pt}
\end{figure}

\subsection{Comparison with a Uniform Prior}
\label{sec:uniform_prior}

\paragraph{Matched--Length Comparison.}
A uniform per-token penalty also shortens trajectories, so penalizing redundant tokens alone does not explain the benefit of a \emph{semantic} prior. To isolate its benefit, we train DeepScaler-1.5B and DLER-1.5B using a uniform-prior penalty, sweeping the penalty coefficient $\beta_u$ over an order of magnitude and tuning hyperparameters at each value. These models provide a strong baseline tracing a clear length--accuracy trade-off. 
For each CIB model, we select the uniform-prior checkpoint closest in average length, breaking ties in favour of accuracy. With average lengths matched to within $\sim 1\%$, the semantic prior improves accuracy by $3.9$ $(\beta^-)$ and $2.8$ $(\beta^+)$ points on DeepScaler-1.5B and by $1.6$ points on DLER-1.5B (Table~\ref{tab:uniform_prior}, left). These gains hold across benchmarks in every comparison and are largest on GPQA, where CIB$_{\beta^-,1.5B}$ is both $6.6$ points more accurate and $24\%$ shorter (Table~\ref{tab:match_avg_len_min} in Appendix~\ref{app:uniform_prior}). No uniform-prior checkpoint reaches the compression level of the 7B semantic-prior model, precluding a matched-length comparison in that regime.


\begin{table}[!ht] \centering \caption{\textbf{Uniform versus semantic prior.} \emph{Left:} accuracy and average completion length over the math benchmarks, comparing CIB models with the matched-length uniform-prior checkpoint ($\beta_u=5\times10^{-4}$). \emph{Right:} redundancy markers over 1432 prompt-paired completions, relative to the most accurate uniform-prior checkpoint ($\beta_u=1\times10^{-4}$). Each arrow denotes uniform\,$\to$\,CIB for the CIB setting specified in the same row. Density is available only for the $\beta^-$ models.} \label{tab:uniform_prior} \footnotesize \setlength{\tabcolsep}{4pt} \begin{tabular}{@{}lclcc@{\quad}cc@{}} \toprule & \multicolumn{4}{c}{Matched-length comparison} & \multicolumn{2}{c}{Marker comparison} \\ \cmidrule(lr){2-5}\cmidrule(l){6-7} & Uniform & \multicolumn{2}{c}{CIB} & $\Delta$ acc.\ $(\uparrow)$ & \multicolumn{2}{c}{Uniform $\to$ CIB} \\ \cmidrule(lr){3-4} Model & Acc.\,/\,len. & Setting & Acc.\,/\,len. & & Per completion $(\downarrow)$ & Per 1k tokens $(\downarrow)$ \\ \midrule DeepScaler-1.5B & $46.0$\,/\,$5224$ & $\beta^-$ & $49.9$\,/\,$5113$ & $\mathbf{+3.9}$ & $64.8\!\to\!34.2$ & $12.2\!\to\!7.5$ \\ & $46.0$\,/\,$5224$ & $\beta^+$ & $48.8$\,/\,$4870$ & $\mathbf{+2.8}$ & $64.8\!\to\!42.2$ & $-$ \\ \cmidrule{1-7}
DLER-1.5B & $47.0$\,/\,$2813$ & $\beta^-$ & $48.6$\,/\,$2789$ & $\mathbf{+1.6}$ & $24.4\!\to\!9.5$ & $10.6\!\to\!5.3$ \\ \bottomrule \end{tabular} \vspace{-10pt} \end{table}


\paragraph{The Two Priors Compress Differently.} The distinction lies in what the penalties can express. A uniform penalty depends only on trajectory length: removing any $k$ tokens reduces the penalty by the same amount, regardless of their content. A content-dependent penalty instead allows the policy to optimize not only how much reasoning to retain, but also which reasoning to retain. To examine this distinction empirically, we measure both redundancy-marker counts and their per-token density. A proportional reduction in markers and tokens would leave marker density unchanged; a decrease in density indicates that markers become less prevalent within the retained reasoning. The right-hand columns of Table~\ref{tab:uniform_prior} show such a decrease. Here, the reference is the most accurate uniform-prior checkpoint, rather than the matched-length checkpoint. 
Relative to the uniform prior, the CIB prior reduces marker counts by $47.2\%$ on DeepScaler-1.5B and $61.1\%$ on DLER-1.5B across $1432$ prompt-paired completions, while reducing marker density per token by $38.5\%$ and $50.0\%$, respectively.
Thus, markers are not merely fewer overall: they are also rarer per token. Relative to the uniform prior, the semantic prior changes the composition of the retained reasoning, rather than simply reducing its length. 
Appendix~\ref{app:uniform_prior} provides the complete $\beta_u$ sweeps, the matched-length selection procedure, marker definitions and per-category counts, and the full paired analysis for both model families.
\section{Conclusions} \label{sec:conclusions}
In this work, we address efficient reasoning in \ac{LLMs} by reframing
``Budget Forcing'' from an information-theoretic perspective. We identified
the ``\paradox''---the structural inconsistency in applying standard
Information Bottleneck principles to transformer architectures---and proposed
a Conditional Information Bottleneck framework to resolve it. This yields a
principled RL objective that subsumes existing length-based penalties as
special cases (Propositions~\ref{prop:linear_penalty} and~\ref{prop:lcpo})
while introducing a semantic token cost based on surprisal under a language
model prior. Our experiments on reasoning benchmarks show that penalizing tokens by their semantic information content yields a more favorable
accuracy--length trade-off than naive length penalties. By tuning $\beta$,
one can traverse the Pareto frontier, achieving up to $\sim48$\% token
reduction with $\lesssim 1.5$\% accuracy degradation. We further find that
stronger priors provide better redundancy estimates, enabling more aggressive
compression with minimal performance loss---while introducing negligible
training overhead ($\sim 6$\% throughput reduction).
These findings suggest that efficient inference requires moving beyond a
``flat tax'' on token count toward metrics that value computation based on
its utility. More broadly, our framework's central contribution is a general
recipe for optimizing reasoning efficiency: by varying the implementations of
the verifier and the prior, practitioners can explore a broad design space
tailored to specific downstream tasks and deployment constraints.

\paragraph{Limitations.}
While \ac{CIB} compresses reasoning traces effectively, two limitations remain. First, compression quality hinges on the reference prior: a weak $Q_\phi$ flattens toward uniform and the semantic cost degenerates into a naive length penalty (Section~\ref{sec:exp_res}), and cross-family priors (e.g.\ a Llama prior for a Qwen policy) add tokenization confounds we do not address---coarser-grained cost functions robust to vocabulary mismatch are a promising direction. Second, $\beta$ and the prior capacity jointly set the operating point on the Pareto frontier; we tuned $\beta$ for the 1.5B prior and did not re-tune at 7B, which may explain the slight accuracy gap, so a systematic study of the $(\beta, Q_\phi)$ interaction could yield further gains.

\bibliographystyle{abbrvnat}
\bibliography{references}

\newpage
\appendix

\section{Conditional Information Bottleneck}
\label{app:CIB}

\paragraph{Preliminaries.} We denote the query, answer, and reasoning traces respectively by the random variables $X, Y$ and $Z$. In this work, we assume that $X,Y,Z \in \gX^*$, where $\gX^*$ is the set of all finite sequences of tokens in the token space $\gX$.  The underlying probability space of these random variables is given by $(\gX^*, \Sigma^*)$, where $\Sigma^*$ is the co-product $\sigma$-algebra using the $\sigma$-algebras on each $\gX^n$. The space $\gX$ is assumed to be discrete.
The entropy of the random variable $X$, the conditional entropy of $Y$ given $X$, and the mutual information between $X$ and $Y$ are defined as:
\begin{align}
\nonumber
H(X)&:=\E_{X\sim P_X}(-\log P(X) ), \\ \nonumber
H(Y|X)&:= \E_{(X,Y)\sim P_{XY}}(-\log P(Y|X) ), \\ \nonumber
I(X;Y) &:= H(X) - H(X|Y).    
\end{align}

\paragraph{Formulation.}
Consider the query-answer pair $(X,Y)$. The function of reasoning is to generate $Z$ such that the LLM probability $\pi_\theta(Y|Z,X)$ is maximized. Budget forcing aims at compressing $Z$. 

In the classical information bottleneck, this problem is formulated as maximizing the information gain $I(Z;Y)$ while minimizing the redundancy with $I(X;Z)$, yielding the following optimization problem:
\[
\min_{\pi_\theta(Z|X)} I(X;Z) - \beta I(Y;Z),
\]
under the Markov assumption $Y\leftrightarrow X \leftrightarrow Z$ and the marginal probability constraint $    \sum_{z}\pi_\theta(z|x) = 1, \forall x
$. In information bottleneck literature, the mutual information $I(X;Z)$ is called rate or complexity, while $I(Y;Z)$ is called relevance or information. We use the terms \textit{minimality} and \textit{sufficiency} in this paper.

In the context of reasoning, the Markov chain $Y\leftrightarrow X \leftrightarrow Z$  does not hold, namely $\pi_\theta(Y|X,Z) \neq \pi_\theta(Y|X)$ because the dense attention mechanism breaks the Markov relation, and the response depends on both the query and the reasoning trace.

There is another subtle issue with the classical information bottleneck setup. The outcome of the optimization problem is $\pi_\theta(z|x)$. The Markov property enables us to generate $y$ based on $z$ using $p(y|z) = \sum_{x}p(y|x)p(x|z)$, which implicitly assumes the knowledge of the conditional probability $p(y|x)$ at decoding time. Without Markov property, this relation cannot be used. Besides, $p(y|x)$ is \textit{not} available at the decoding time, and it is unknown in the context of reasoning. The goal of reasoning trace $z$ is to enable the simulation of $p(y|x)$ using $\pi_\theta(z|x)\pi_\theta(y|z,x)$, which points toward a connection with \textit{channel simulation} literature \cite{Li2024-channelsimulation}.  

We would like to maximize the information gain of the reasoning trace $Z$ with the knowledge of the query $X$. We can measure the gain using the conditional mutual information $I(Y;Z|X)$. 
The conditional information bottleneck version is as follows: 
\[
\min_{P(Z|X,Y)} I(X;Z) - \beta I(Y;Z|X),
\]
where we need the following marginal probability constraints to be satisfied:
\begin{align*}
    \sum_{z}P(z|x,y) = 1, \forall x, y.
\end{align*}
If we cast the problem as a maximization problem, the final optimization problem in terms of $P(Z|X,Y)$ is as follows:
\begin{align}
\max_{P(Z|X,Y)} & \sum_{x,y,z} P(x,y) P(z|x,y) \log P(y|x,z) - \beta \sum_{x,z} P(x)P(z|x)\log \frac{P(z|x)}{P(z)} \nonumber \\
    \text{s.t.} &\sum_{z}P(z|x,y) = 1, \forall x, y. \label{eq:CIB_v0}
\end{align}
The dependence on $P(z|x,y)$ is implicit in various distributions  like $P(z|x), p(z)$ and $p(y|x,z)$, and we can solve this problem using an iterative algorithm, similar to Blahut-Arimoto algorithm, proposed in \cite{tishby1999information}. Given that $z$ is from the space of reasoning traces, it is not tractable to use the same approach. 

We would like to remark that the information bottleneck problem is an instance of lossy compression under log-loss distortion metric. In this sense, the conditional information bottleneck is akin to lossy compression with side information, which was studied by Wyner and Ziv \cite{wyner1976rate} under different settings. 


\paragraph{Practical Implementation.}
The information bottleneck optimization in deep learning is directly intractable, and the approximate bounds are used for training such variational information bottleneck \cite{Alemi2017-variationalIB}. In practice, for training LLMs, we do not directly optimize over $P(Z|Y,X)$ but rather optimize the model parameter $\theta$.

The parameter $\theta$ controls the two probabilities $ \pi_\theta(z|x)$ and $\pi_\theta(y|x,z)$, both the inference part of the AR generative model (instead of $P(Z|Y,X)$. We solve the following optimization problem:
\begin{align}
\max_{\theta}\quad & \gL_{\text{CIB}} (\theta)  = I(Y;Z|X) - \beta I(X;Z)\nonumber\\
\quad  \text{s.t.} & \quad  P(y|x) = \sum_{z}\pi_\theta(y|x,z)\pi_\theta(z|x), \quad\forall x,y.
\end{align}
The last constraint should be satisfied to have a valid probability distribution. Since all the probabilities $\pi_\theta(\cdot|\cdot)$ are parametrized to sum up to one, we do not need to explicitly add the constraint. 

\color{black}
Note that we can write $P(z|x,y){P(y|x)} = {\pi_\theta(y|x,z)\pi_\theta(z|x)}$. In other words, we  can obtain a valid $P(z|x,y)$ from $\pi_\theta(y|x,z)$ and $\pi_\theta(z|x)$, and vice versa. Therefore, the above optimization problem is just a reparameterization of the original information bottleneck problem and yields the same optimal value.
\color{black}

There are some challenges with the above optimization problem. First, the conditional distribution $P(y|x)$ is unknown in general, and we have access to it only through the samples. Second, we should approximate the information theoretic quantities, namely the sufficiency term $I(Y;Z|X)$ and the minimality term $I(X;Z)$. We try to address these challenges below.

\paragraph{Sufficiency term.}
Consider the first term in the objective function. We can write it as a function of the optimization parameters $\pi_\theta(y|x,z), \pi_\theta(z|x)$ as follows: 
  \begin{align*}
       I(Y;Z|X)  & = \sum_{x,y,z} P(x,y) P(z|x,y)\log\frac{\pi_\theta(y|x,z)}{P(y|x)} \\
         & = \sum_{x,y,z} P(x,y) \pi_\theta(z|x) \frac{\pi_\theta(y|x,z)}{P(y|x)} \log\frac{\pi_\theta(y|x,z)}{P(y|x)} \\
         & \geq \sum_{x,y,z} P(x,y) \pi_\theta(z|x) \log\frac{\pi_\theta(y|x,z)}{P(y|x)} 
  \end{align*}
  where we used $x\log x \geq \log x$ in the last step. Therefore, we can  maximize the lower bound and approximate it further using the query-answer samples $(x_i,y_i)$. The first term of the optimization problem is:
  \begin{equation}
      \sum_{i=1}^m \E_{Z\sim \pi_\theta(\cdot|x_i)} [\log \pi_\theta(y_i|x_i,Z)].
  \end{equation}

We can also approximate the bound using variational approximation of \cite{Alemi2017-variationalIB}. We introduce a \textit{verifier model} $Q_\rho(y|x,z)$ as variational parameter:
  \begin{align*}
       I(Y;Z|X)  & = \sum_{x,y,z} P(x,y) P(z|x,y)\log\frac{\pi_\theta(y|x,z)}{P(y|x)} \\
       & = \sum_{x,y,z} P(x,y) P(z|x,y)\log\frac{\pi_\theta(y|x,z)Q_\rho(y|z,x)}{P(y|x)Q_\rho(y|z,x)} \\
       & = \sum_{x,y,z} P(x,y) P(z|x,y)\log\frac{Q_\rho(y|z,x)}{P(y|x)} + \E_{(X,Z)} D_{KL}(\pi_\theta(\cdot|X,Z)\Vert Q_\rho(\cdot|X,Z)) \\
       & \geq  \sum_{x,y,z} P(x,y) P(z|x,y)\log\frac{Q_\rho(y|z,x)}{P(y|x)}\\
         & = \sum_{x,y,z} P(x,y) \pi_\theta(z|x) \frac{\pi_\theta(y|x,z)}{P(y|x)} \log\frac{Q_\rho(y|z,x)}{P(y|x)} 
  \end{align*}
Throughout the paper, we assume that there is always a unique answer to each query. Using this assumption, we can lower bound the last step as follows:
  \begin{align*}
         \sum_{x,y,z} P(x,y)  &  \pi_\theta(z|x) \frac{\pi_\theta(y|x,z)}{P(y|x)} \log\frac{Q_\rho(y|z,x)}{P(y|x)} \\
         & \geq \sum_{x,y=y_{\text{true}}(x),z} P(x,y) \pi_\theta(z|x) \log{Q_\rho(y|z,x)},
  \end{align*}
  where we used the assumption that $P(y|x)=\delta(y-y_{\text{true}}(x))$. 
Finally, we can maximize the following objective function for the sufficiency term:
    \begin{align*}
      \sum_{i=1}^m \E_{Z\sim \pi_\theta(\cdot|x_i)} [\log Q_\rho(y_i|x_i,Z)].
  \end{align*}

\paragraph{Minimality term.} For the minimality term, we use a variational approximation to find a variational bound similar to \cite{Alemi2017-variationalIB}. This upper bound combined with the above lower bound provides a general lower bound on the conditional information bottleneck objective which we try to maximize. First note that:
\[
I(X;Z)   = \sum_{x,z} P(x) \pi_\theta(z|x)\log\frac{\pi_\theta(z|x)}{P(z)}.
\]
There is a dependence on $P(z)$ which requires marginalization over $x$ and is intractable. The derivation of the variational lower bound is quite standard:
  \begin{align*}
       I(X;Z)  & = \sum_{x,z} P(x) \pi_\theta(z|x)\log\frac{\pi_\theta(z|x)}{P(z)} \\
         & = \sum_{x,z} P(x) \pi_\theta(z|x)\log\frac{\pi_\theta(z|x)}{Q_\phi(z)} - D_{KL}(P(Z)\Vert Q_\phi(Z)) \\
         & \leq \sum_{x,z} P(x) \pi_\theta(z|x)\log\frac{\pi_\theta(z|x)}{Q_\phi(z)}.
  \end{align*}
The variational distribution $Q_\phi(\cdot)$ is supposed to capture the distribution over reasoning traces without conditioning on the prompt. 

The final optimization problem consists of finding a policy $\pi_\theta(z|x)$ that maximizes the returns defined from the above approximate bounds using $Q_\phi$ and $Q_\rho$.
\paragraph{Marginal probability constraint.} Let's consider again the constraint for the conditional information bottleneck:
\[
 P(y|x) = \sum_{z}\pi_\theta(y|x,z)\pi_\theta(z|x).
\]
As we mentioned above, the conditional probability is unknown. Now, assume that for each query there is a unique answer: $P(y|x) = \delta(y-y_{\text{true}}(x))$. In this case, the marginal probability constraint holds only if $\pi_\theta(y|x,z)$ is also equal to $\delta(y-y_{\text{true}}(x))$, namely:
\begin{equation}
    \pi_\theta(y_{\text{true}}(x)|x,z) = 1,
\end{equation}
We do not need explicitly include this constraint in the optimization problem, because it amounts to maximizing the probability of the correct answer under $\pi_\theta(y|x,z)$ which is already part of the optimization problem. 

\section{Additional Analysis}\label{app:add_res}

\begin{figure}[t]
\centering
  \begin{minipage}{0.45\linewidth}
    \centering
    \includegraphics[width=\linewidth]{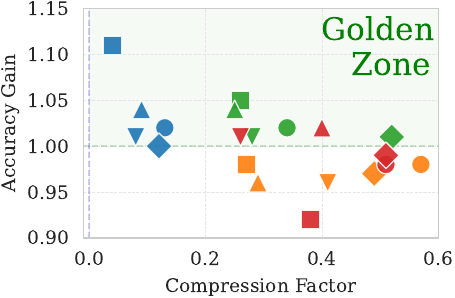}
    \label{fig:efficiency_gain}
  \end{minipage}
  \begin{minipage}{0.45\linewidth}
    \centering
    \includegraphics[width=\linewidth]{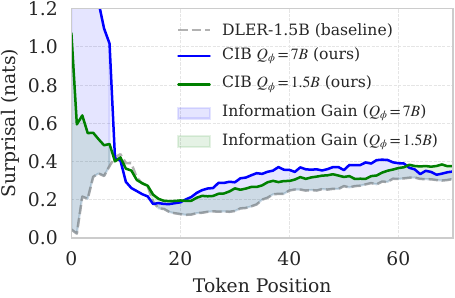}
    \label{fig:info_density}
  \end{minipage}\hfill
  \caption{
  \textbf{Left: Efficiency Gain.} Efficiency gain of \ac{CIB} across diverse benchmarks and models. Points falling in the upper half-plane (``\textcolor{DarkGreen}{Golden Zone}") exhibit strictly superior efficiency, achieving higher information density with reduced computational cost. Colors denote models: CIB$^{DLER}_{Q_\phi\!=\!1.5B}$ ({\color{dlercib15}$\bullet$}), CIB$^{DLER}_{Q_\phi\!=\!7B}$ ({\color{dlercib7}$\bullet$}), CIB$^{DeepSc.}_{Q_\phi\!=\!1.5B}$ ({\color{deepcib15}$\bullet$}), CIB$^{DeepSc.}_{Q_\phi\!=\!7B}$ ({\color{deepcib7}$\bullet$}). Shapes denote tasks: MATH500 ($\bullet$), AIME24 ($\blacksquare$), AIME25 ($\blacktriangle$), Minerva ($\blacklozenge$), Olympiad ($\blacktriangledown$).
  \textbf{Right: Information Density Profile.} Token-wise surprisal evaluated against the baseline language prior. Lower surprisal indicates predictable, low-information filler.
  }
  \label{fig:side_side_2}
\end{figure}

\paragraph{Efficiency gain.}
To quantify the trade-off between reasoning performance and computational cost, we define two metrics: the \emph{Compression Factor} $=1-\ell_{\text{CIB}}/\ell_{\text{base}}$, measuring the relative reduction in completion length, and the \emph{Accuracy Ratio} $=\mathcal{A}_{\text{CIB}}/\mathcal{A}_{\text{base}}$, normalizing performance against the baseline. Figure~\ref{fig:side_side_2} (left) plots both across architectures and benchmarks. The upper half-plane (``\textcolor{DarkGreen}{Golden Zone}'') holds models that are strictly superior in both speed and accuracy, while the bottom-right offers significant speedups where partial accuracy degradation is acceptable. Filtering for the Golden Zone therefore selects models that are ``smarter'' and faster rather than merely truncated. Note that each model's gains are normalized to its own baseline. For visual clarity, the figure focuses on DLER-1.5B and DeepScaleR across the five math benchmarks and both prior sizes.

\paragraph{Information density.}
To understand \emph{how} \ac{CIB} compresses, we analyze the token-wise surprisal $-\log p(z_t \mid z_{<t}, x)$ measured against a frozen reference model. In Figure~\ref{fig:side_side_2} (right), the baseline profile (dashed gray) shows valleys of low surprisal ($\approx0.1$ nats), i.e.\ highly predictable sequences carrying negligible unique information, whereas the \ac{CIB} profiles (solid blue and green) maintain a higher floor ($\gtrsim0.2$ nats). By penalizing cumulative low-utility transitions, the objective acts as a high-pass semantic filter that excises predictable filler while preserving high-entropy peaks. \ac{CIB} thus compresses not by random truncation but by raising the average information rate, distilling reasoning traces to their essential computational content.

\section{CoT Qualitative Comparison: Pruning Reasoning Redundancy} \label{app:qualitative_analysis}

To examine the nature of the compression induced by our semantic prior, we visualize qualitative differences between baseline traces and \ac{CIB}-generated traces across a range of reasoning tasks (see Figures \ref{fig:cot_1}--\ref{fig:cot_3}). We find that the semantic prior does not merely truncate outputs; instead, it changes \emph{which} computations are expressed in the trace. Concretely, by imposing an information-cost on the reasoning tokens (via the prior surprisal) while preserving task success through the sufficiency objective, \ac{CIB} penalizes computation that offers low marginal information regarding the target $Y$. This effect manifests through three recurring mechanisms:

\begin{itemize}

    \item Induction of Algorithmic Generalization.\\ 
    Perhaps most notably, the information bottleneck biases the model toward theoretically superior solution paths. In geometric reasoning (Figure~\ref{fig:cot_1}), while the baseline defaults to brute-force coordinate calculations via the Pythagorean theorem, the \ac{CIB} model converges on a concise trigonometric identity ($\sin T = \cos R$). This suggests that minimizing the redundant computation under the semantic prior of the reasoning trace naturally selects for abstract, elegant proofs, as these represent the most compressed description of the transformation from prompt $X$ to answer $Y$.
    
    \item Suppression of Stochastic Exploration and Verification Bloat.\\ 
    Baseline models typically adopt a high-entropy strategy characterized by reasoning redundancy, utilizing conversational scaffolding and unstructured exploration. For instance, in arithmetic search tasks (Figure~\ref{fig:cot_2}), the baseline explicitly calculates invalid candidates (e.g., $98^3$) before testing the correct one. Similarly, in constraint satisfaction problems (Figure~\ref{fig:cot_3}), the baseline engages in tautological self-checks (e.g., verifying that positive lengths satisfy $x>0$). \ac{CIB} eliminates these low-utility branches. By penalizing the cumulative surprisal of the chain, the policy shifts from ``exploratory thinking" to ``efficient execution," treating valid derivations as terminal states rather than triggering redundant self-doubt loops.

    \item Semantic Filtering  of Syntactic Artifacts.\\ 
    \ac{CIB} acts as a semantic filter that separates essential state information from syntactic artifacts. As shown in Figure~\ref{fig:cot_1}, when presented with raw code metadata (Asymptote), the baseline expends significant budget on ``verbal parsing"—reading the code aloud without progressing the state. The compressed policy bypasses this verbalization, extracting the underlying geometric conditions directly.

\end{itemize}

\begin{figure*}[!ht]
    \centering
    
    \begin{tcolorbox}[colback=white, colframe=black, title=\textbf{Input Prompt}]
    \textbf{Problem:} In $\triangle{RST}$, shown, $\sin{R}=\frac{2}{5}$. What is $\sin{T}$? \\
    \textit{[Asymptote Code Segment]:} \texttt{pair R,S,T; T = (0,0); S = (2,0); R = (2,sqrt(21)); ...}
    \end{tcolorbox}
    \vspace{0.1cm}
    
    \begin{tcolorbox}[colback=gray!5, colframe=gray!40, title=\textbf{Baseline Response (Standard CoT) - 1685 tokens}]
    \bloat{Okay, so I have this triangle RST... The Asymptote code draws a right-angled triangle...} \bloat{The Asymptote code defines points T at (0,0)...} \bridge{Triangle RST is right-angled at S.} 
    \bloat{So, angle at S is 90 degrees... Wait, the label is on (R+T)/2...} \bridge{RT is the hypotenuse, length 5.}
    
    \bloat{Let me assign the sides... Given sin R = 2/5...} \bridge{sin R = ST / RT = 2/5. RT is 5, so ST = 2.}
    \bloat{Now, since it's a right-angled triangle at S, we can use Pythagoras...} \bridge{RS$^2$ + 2$^2$ = 5$^2$... RS = $\sqrt{21}$.}
    
    \bloat{Now, we need to find sin T... Wait, no. Wait, angle T is at point T... Wait, maybe I...}
    \bridge{sin T = opposite / hypotenuse... RS / RT.} \bridge{sin T = $\sqrt{21}$/5.}
    \end{tcolorbox}
    \vspace{0.1cm} 
    
    \begin{tcolorbox}[colback=blue!5, colframe=blue!40, title=\textbf{Ours: \ac{CIB} - 154 tokens}]
    First, observe that \bridge{triangle RST is a right-angled triangle} at point S. Given that \bridge{sin R = 2/5.}
    
    In a right-angled triangle, the sum of angles R and T is 90 degrees, so angle T is complementary to angle R. Using the sine of complementary angles, \bridge{sin T = cos R.}
    
    Since sin R = 2/5, cos R = $\sqrt{1 - (2/5)^2} = \sqrt{21/25} = \sqrt{21}/5$.
    Therefore, \bridge{sin T = $\sqrt{21}$/5.}
    \end{tcolorbox}

    \caption{\textbf{Qualitative comparison on geometry reasoning.} \textbf{Top:} Prompt. \textbf{Middle:} the Baseline trace  is dominated by redundant ``verbal parsing" of the input code and repetitive self-correction loops (highlighted in red). \textbf{Bottom:} The CIB trace successfully filters this syntactic noise. Notably, the information constraint induces a shift in strategy: CIB bypasses the lengthy coordinate calculation favored by the baseline, converging instead on a concise trigonometric identity.}
    \label{fig:cot_1}
\end{figure*}

\begin{figure*}[t]
    \centering
    
    \begin{tcolorbox}[colback=white, colframe=black, title=\textbf{Input Prompt}]
    \textbf{Problem:} When the two-digit integer $``\text{AB}"$ is cubed, the value is $912,673$. What is $A + B$?
    \end{tcolorbox}
    \vspace{0.1cm}
    
    \begin{tcolorbox}[colback=gray!5, colframe=gray!40, title=\textbf{Baseline Response (Standard CoT) - 965 tokens}]
    \bloat{Okay, so I have this problem... I need to find A + B... Hmm... Wait, let me think.} \bridge{AB is a two-digit integer... between 10 and 99.} 
    
    \bloat{First, maybe I should find the cube root... I know that 100$^3$ is 1,000,000... Maybe 99$^3$ is about 970299... 98$^3$ is 98 * 98 * 98. Let me calculate that.}
    \bloat{9604 * 98... 960400 - 19208 = 941192. So 98$^3$ is 941192, which is larger than 912673.}
    
    \bloat{What about 97$^3$? Let me compute that.} \bridge{97 * 97 is 9409... 9409 * 97 = 912,673.} \bloat{Oh, wait! 97 cubed is 912673. So AB is 97.}
    
    \bloat{Wait wait, let me double-check. 97 cubed is 912673... Yes, correct.}
    \bridge{So AB is 97, so A is 9, B is 7.} \bridge{Therefore, A + B is 9 + 7 = 16.}
    \end{tcolorbox}
    \vspace{0.1cm} 
    
    \begin{tcolorbox}[colback=blue!5, colframe=blue!40, title=\textbf{Ours: \ac{CIB} - 210 tokens}]
    To determine the value of A + B, we start by noting that \bridge{the two-digit integer AB is a number between 10 and 99.} When AB is cubed, the result is 912,673.
    
    Taking the cube root of 912,673 gives an approximate value. The cube root of 912,673 is approximately 97, since \bridge{97$^3$ = 912,673.}
    
    Thus, \bridge{AB is 97.} The digits A and B are 9 and 7, respectively. Adding them together, \bridge{A + B equals 9 + 7 = 16.}
    \end{tcolorbox}

    \caption{\textbf{Qualitative comparison on arithmetic search.} \textbf{Top:} Prompt. \textbf{Middle:} Baseline model engages in inefficient trial-and-error, explicitly calculating the incorrect candidate $98^3$ (highlighted in red) and engaging in redundant self-verification loops. \textbf{Bottom:} the \ac{CIB} model (bottom) suppresses this exploratory computation, converging directly on the correct candidate ($97$) and reducing the token count by $\sim$78\% without loss of accuracy.}
    \label{fig:cot_2}
\end{figure*}

\begin{figure*}[t]
    \centering
    
    \begin{tcolorbox}[colback=white, colframe=black, title=\textbf{Input Prompt}]
    \textbf{Problem:} Two sides of a triangle are each $8$ units long. If the third side has a whole number length, what is the greatest possible perimeter, in units, for the triangle?
    \end{tcolorbox}
    \vspace{0.1cm}
    
    \begin{tcolorbox}[colback=gray!5, colframe=gray!40, title=\textbf{Baseline Response (Standard CoT)}]
    \bloat{Okay, so I have this problem... Hmm, let me think about how to approach this.} \bridge{Triangle inequality theorem... sum of any two sides $>$ third.}
    
    \bridge{$8 + 8 > c$. That means $16 > c... c$ can be at most 15.}
    
    \bloat{But wait, wait. Let me think again... The triangle inequality also requires... $8 + c > 8$. Well, that simplifies to $c > 0$, which is always true...}
    
    \bloat{But wait, let me double-check. Let's plug c = 15 into the triangle inequalities... 1. $8 + 8 > 15$? Yes... 2. $8 + 15 > 8$? Yes... So all the inequalities are satisfied.}
    
    \bloat{Wait, but hold on. Let me make sure I didn't miss any constraints... c can be 15...}
    \bridge{Perimeter is 8 + 8 + 15 = 31 units.}
    \end{tcolorbox}
    \vspace{0.1cm} 
    
    \begin{tcolorbox}[colback=blue!5, colframe=blue!40, title=\textbf{Ours: \ac{CIB}}]
    To determine the greatest possible perimeter... we apply the \bridge{triangle inequality theorem.} The sum of any two sides must exceed the third side. Let the third side be `$x$'.
    
    The inequalities are:
    1. \bridge{$8 + 8 > x \rightarrow x < 16$}
    2. $8 + x > 8 \rightarrow x > 0$ (always true)
    
    Thus, x must be less than 16 and a positive integer. \bridge{The maximum integer less than 16 is 15.}
    Therefore, the \bridge{maximum perimeter is 8 + 8 + 15 = 31.}
    \end{tcolorbox}

    \caption{\textbf{Qualitative comparison on constraint satisfaction.}\textbf{Top:} Prompt. \textbf{Middle:} Baseline trace is characterized by ``verification bloat." Despite correctly deriving the constraint ($c < 16$) early on, the model expends tokens checking tautologies ($8+c > 8$) and re-verifying its own conclusions (highlighted in red). \textbf{Bottom:} \ac{CIB} trace (bottom) retains the necessary constraint logic but eliminates the redundant self-auditing loops, trusting the derivation immediately.}
    \label{fig:cot_3}
\end{figure*}
 
\clearpage
\section{Uniform Prior Ablation: Length and Reasoning Composition}
\label{app:uniform_prior}

Both a uniform per-token penalty and the CIB semantic prior discourage verbose trajectories. Suppressing tokens that add cost without improving accuracy therefore does not, by itself, explain the semantic prior's benefit. We investigate this difference in the DeepScaler-1.5B and DLER-1.5B families through three analyses: a uniform-prior $\beta_u$ sweep with per-value hyperparameter tuning, an accuracy comparison at comparable average completion length, and a paired trajectory analysis over 1432 samples across five benchmarks.

\paragraph{Length-dependent versus content-dependent cost.}
Under a uniform prior, trajectory cost depends only on token count: removing any $k$ tokens yields the same cost reduction, regardless of their content. A semantic prior makes cost content-dependent, allowing the policy to optimize not only completion length but also the reasoning it retains. We test whether this distinction is reflected in trajectory composition by measuring redundant-reasoning markers and their density per token. Proportional shortening would leave marker density unchanged; a lower density indicates a change in composition beyond the reduction in length.

Surprisal weighting may also penalize load-bearing reasoning tokens. However, the penalty is balanced against the accuracy reward, so retaining such tokens can justify their cost. Higher accuracy at comparable average length provides evidence against indiscriminate removal of useful reasoning (\autoref{tab:match_avg_len_min} and \autoref{tab:avg_len_dler}).

\subsection{DeepScaler-1.5B}
\label{subsec:deepscaler_uniform}

Here, $CIB_{\beta^{-},1.5B}$ and $CIB_{\beta^{+},1.5B}$ denote the two DeepScaler-1.5B semantic-prior models reported in the paper.

\paragraph{Uniform-prior tuning.}
We sweep $\beta_u$ over one order of magnitude (Table~\ref{tab:deepscaler_bench}). Below this range, the penalty is negligible; above it, the penalty dominates and training collapses within the first few iterations. The uniform prior was more sensitive to other hyperparameters than the semantic prior. For each $\beta_u$, we therefore ran a small grid search over learning rate, KL-regularization strength, and other hyperparameters, and report the best checkpoint. Training used a maximum generation length of 8192 tokens, as in the semantic-prior runs; evaluation followed the paper's protocol, allowing up to 32K new tokens.

\begin{table}[!ht]
\centering
\caption{\textbf{DeepScaler-1.5B uniform-prior sweep.} Accuracy / average completion length on five reasoning benchmarks.}
\label{tab:deepscaler_bench}
\resizebox{\textwidth}{!}{
\begin{tabular}{lcccccc}
\toprule
Setting & AIME24 & AIME25 & GPQA & MATH500 & Olympiad & mean \\
\midrule
\multicolumn{7}{l}{\emph{Uniform prior, $\beta_u$ sweep}} \\
\midrule
$\beta_u = 5 \times 10^{-5}$ & 34.6 / 9261 & 25.8 / 8715 & 35.9 / 7446 & 85.8 / 3067 & 54.4 / 5871 & 47.3 / 6872 \\
$\beta_u = 1 \times 10^{-4}$ & 37.1 / 9025 & 27.7 / 8582 & 32.3 / 7257 & 88.6 / 3193 & 54.6 / 5756 & 48.1 / 6763 \\
$\beta_u = 3 \times 10^{-4}$ & 34.6 / 9050 & 26.5 / 8179 & 31.8 / 7528 & 86.4 / 2950 & 55.1 / 5613 & 46.9 / 6664 \\
$\beta_u = 5 \times 10^{-4}$ & 36.9 / 6473 & 24.8 / 7002 & 31.8 / 6051 & 84.2 / 2270 & 52.2 / 4322 & 46.0 / 5224 \\
\bottomrule
\end{tabular}
}
\end{table}

\subsubsection{Matched-length comparison}
\label{app:match_len}

We selected the uniform checkpoint closest in average completion length to the semantic-prior models, choosing the highest-accuracy checkpoint among those at comparable length. The selected checkpoint, $\beta_u = 5 \times 10^{-4}$ (Table~\ref{tab:deepscaler_bench}), averages 5224 tokens, compared with 5113 for $CIB_{\beta^{-},1.5B}$ and 4870 for $CIB_{\beta^{+},1.5B}$. It is thus closely matched to the former and 7\% longer than the latter. No uniform checkpoint reached the 7B semantic-prior model's average length of 3529 tokens, so we exclude that model from the matched-length comparison. We report the results in Table~\ref{tab:match_avg_len_min}

\begin{table}[!ht] \centering \caption{\textbf{DeepScaler-1.5B matched-length comparison with $CIB_{\beta^{-},1.5B}$.} The uniform baseline uses $\beta_u = 5 \times 10^{-4}$, the highest-accuracy $(\uparrow)$ uniform checkpoint at comparable average length $(\downarrow)$.} \label{tab:match_avg_len_min} \begin{tabular}{lc|cc|cc} \toprule Task & Uniform & $CIB_{\beta^{-},1.5B}$ & $\Delta$ Acc. & $CIB_{\beta^{+},1.5B}$ & $\Delta$ Acc. \\ & Acc. / Avg. Len. & Acc. / Avg. Len. & $(\uparrow)$ & Acc. / Avg. Len. & $(\uparrow)$ \\ \midrule AIME24 & 36.9 / 6473 & 40.2 / 7298 & +3.3 & 38.2 / 6937 & +1.3 \\ AIME25 & 24.8 / 7002 & 27.8 / 6866 & +3.0 & 27.6 / 6757 & +2.8 \\ GPQA & 31.8 / 6051 & 38.4 / 4571 & +6.6 & 35.5 / 4429 & +3.7 \\ MATH500 & 84.2 / 2270 & 87.6 / 2359 & +3.4 & 87.2 / 2017 & +3.0 \\ Olympiad & 52.2 / 4322 & 55.5 / 4472 & +3.3 & 55.3 / 4211 & +3.1 \\ \midrule Mean & 46.0 / 5224 & \textbf{49.9} / 5113 & \textbf{+3.9} & 48.8 / \textbf{4870} & +2.8\\ \bottomrule \end{tabular} \end{table}

The tuned uniform penalty provides a strong baseline, but the semantic prior achieves a better accuracy--length trade-off. At slightly shorter average lengths, $CIB_{\beta^{-},1.5B}$ and $CIB_{\beta^{+},1.5B}$ improve mean accuracy by $3.9$ and $2.8$ points, respectively, with gains on all five benchmarks. AIME24 and AIME25 each contain only 30 problems, so their individual differences should be interpreted cautiously. Gains persist on the three larger benchmarks and are largest on GPQA, where $CIB_{\beta^{-},1.5B}$ is $6.6$ points more accurate and 24\% shorter.

\subsubsection{Paired trajectory analysis}

We compare completions paired by benchmark and sample index over 1432 samples. Unlike the matched-length comparison, this analysis uses the highest-accuracy uniform checkpoint, $\beta_u = 1 \times 10^{-4}$. Its longer completions produce a wider token gap than in \autoref{app:match_len}, making length-normalized marker density essential for distinguishing changes in composition from shortening alone.

We use four marker categories as proxies for redundant reasoning loops:
\begin{itemize}
    \item \textbf{backtrack:} \emph{Wait, Hold on, Actually, Hmm};
    \item \textbf{reverify:} \emph{Let me check, Let me verify, Let me double-check, Let me re-examine, Let me recompute};
    \item \textbf{alternative:} \emph{Alternatively, Another way, Another approach, Or maybe};
    \item \textbf{self-correct:} \emph{That's wrong, I made a mistake, Let me reconsider, I think I made*}.
\end{itemize}
For each completion, we count markers by category and in total, then compute paired differences relative to the uniform completion. Entries in the trajectory-analysis tables are rounded independently; differences and totals may therefore differ by one unit in the last digit from arithmetic on the displayed values. We report the results in Table~\ref{tab:mean_mark_doc_min}.

\begin{table}[!ht]
\caption{\textbf{DeepScaler-1.5B average markers count per completion $(\downarrow)$: uniform prior versus $CIB_{\beta^{-},1.5B}$.}}
\label{tab:mean_mark_doc_min}
\centering
\begin{tabular}{lcc|cc|cc}
\toprule
Metric & n samples & Uniform & $CIB_{\beta^{-},1.5B}$ & $\Delta$ & $CIB_{\beta^{+},1.5B}$ & $\Delta$ \\
& & & & $(\downarrow)$ & & $(\downarrow)$ \\
\midrule
backtrack & 1432 & 40.19 & 20.27 & -19.92 & 25.92 & -14.27 \\
reverify & 1432 & 1.21 & 0.77 & -0.45 & 0.81 & -0.41 \\
alternative & 1432 & 22.43 & 12.59 & -9.84 & 14.70 & -7.73 \\
self-correct & 1432 & 1.02 & 0.60 & -0.42 & 0.76 & -0.27 \\
\midrule
All markers & 1432 & 64.85 & \textbf{34.22} & \textbf{-30.63} & 42.18 & -22.67 \\
\bottomrule
\end{tabular}
\end{table}

Both semantic-prior variants reduce every marker category (Table~\ref{tab:mean_mark_doc_min}). For $CIB_{\beta^{-},1.5B}$, the total marker count falls by 47.2\%. Marker density also falls, from 12.2 to 7.5 per 1000 tokens (\autoref{tab:len_vs_mark_red}). Thus, the reduction in markers exceeds what proportional shortening alone would explain.

\begin{table}[!ht]
\centering
\caption{\textbf{DeepScaler-1.5B marker reduction relative to length for $CIB_{\beta^{-},1.5B}$.}}
\label{tab:len_vs_mark_red}
\begin{tabular}{lcccc}
\toprule
Metric & n samples & Uniform & $CIB_{\beta^{-},1.5B}$ & \% change \\
\midrule
n\_markers & 1432 & 64.8 & \textbf{34.2} & -47.2 \\
markers\_per\_1k\_tokens & 1432 & 12.2 & \textbf{7.5} & -38.4 \\
\bottomrule
\end{tabular}
\end{table}

\subsection{DLER-1.5B}
\label{subsec:dler_uniform}

In this section, $CIB_{\beta^{-},1.5B}$ denotes the DLER-1.5B semantic-prior model reported in the paper. We repeat the DeepScaler procedure with the same training and evaluation setup. DLER's shorter baseline chains of thought allow us to explore slightly higher values of $\beta_u$.

\begin{table}[!ht]
\centering
\caption{\textbf{DLER-1.5B uniform-prior sweep.} Accuracy / average completion length on five reasoning benchmarks.}
\label{tab:acc_compl_dler}
\resizebox{\textwidth}{!}{
\begin{tabular}{lcccccc}
\toprule
Setting & AIME24 & AIME25 & GPQA & MATH500 & Olympiad & mean \\
\midrule
\multicolumn{7}{l}{\emph{Uniform prior, $\beta_u$ sweep}} \\
\midrule
$\beta_u = 9 \times 10^{-5}$ & 36.5 / 3508 & 25.3 / 3386 & 38.9 / 3260 & 85.6 / 2093 & 54.0 / 2726 & 48.1 / 2995 \\
$\beta_u = 3 \times 10^{-4}$ & 33.5 / 3431 & 24.6 / 3193 & 38.9 / 3189 & 86.2 / 1981 & 54.0 / 2645 & 47.4 / 2888 \\
$\beta_u = 5 \times 10^{-4}$ & 33.8 / 3376 & 25.0 / 3087 & 37.9 / 3125 & 85.6 / 1853 & 52.6 / 2623 & 47.0 / 2813 \\
$\beta_u = 8 \times 10^{-4}$ & 35.4 / 2772 & 23.0 / 2627 & 35.4 / 2890 & 84.8 / 1723 & 50.7 / 2232 & 45.9 / 2449 \\
\bottomrule
\end{tabular}
}
\end{table}

\subsubsection{Matched-length comparison}

The same selection rule yields the uniform checkpoint at $\beta_u = 5 \times 10^{-4}$ (Table~\ref{tab:acc_compl_dler}), averaging 2813 tokens versus 2789 for the semantic-prior model: a difference of less than 1\%. As with DeepScaler, no uniform checkpoint reached the 7B semantic-prior model's average length, precluding a matched-length comparison at that scale. We report the results in Table~\ref{tab:avg_len_dler}.

\begin{table}[!ht]
\centering
\caption{\textbf{DLER-1.5B matched-length comparison with $CIB_{\beta^{-},1.5B}$.} The uniform baseline uses $\beta_u = 5 \times 10^{-4}$.}
\label{tab:avg_len_dler}
\begin{tabular}{lccc}
\toprule
Task & Uniform & $CIB_{\beta^{-},1.5B}$ & $\Delta$ Acc.  \\
& Acc. / Avg. Len. & Acc. / Avg. Len. & $(\uparrow)$ \\
\midrule
AIME24 & 33.8 / 3376 & 35.4 / 3198 & +1.6 \\
AIME25 & 25.0 / 3087 & 25.3 / 3260 & +0.3 \\
GPQA & 37.9 / 3125 & 39.4 / 3183 & +1.5 \\
MATH500 & 85.6 / 1853 & 88.4 / 1687 & +2.8 \\
Olympiad & 52.6 / 2623 & 54.3 / 2616 & +1.7 \\
\midrule
Mean & 47 / 2813 & \textbf{48.6 / 2789} & \textbf{+1.6} \\
\bottomrule
\end{tabular}
\end{table}

At average completion lengths matched to within 1\%, the semantic prior improves accuracy on all five benchmarks, with a mean gain of $1.6$ points (Table~\ref{tab:avg_len_dler}). As above, the gains on GPQA, MATH500, and Olympiad are less subject to the small-sample caveat for AIME24 and AIME25.

\subsubsection{Paired trajectory analysis}

We apply the same pairing procedure and marker categories to 1432 samples, again using the highest-accuracy uniform checkpoint for the trajectory analysis. Results are reported in \autoref{tab:mean_mark_doc_plus_dler}--\ref{tab:len_vs_mark_red_dler}.

\begin{table}[!ht]
\centering
\caption{\textbf{DLER-1.5B mean marker counts per completion: uniform prior versus $CIB_{\beta^{-},1.5B}$.}}
\label{tab:mean_mark_doc_plus_dler}
\begin{tabular}{lcccc}
\toprule
Metric & n samples & Uniform & $CIB_{\beta^{-},1.5B}$ & $\Delta$ \\
\midrule
backtrack & 1432 & 21.19 & 7.40 & -13.79 \\
reverify & 1432 & 0.23 & 0.20 & -0.03 \\
alternative & 1432 & 2.91 & 1.75 & -1.16 \\
self-correct & 1432 & 0.10 & 0.10 & -0.00 \\
\midrule
All markers & 1432 & 24.43 & \textbf{9.45} & \textbf{-14.98} \\
\bottomrule
\end{tabular}
\end{table}

\begin{table}[!ht]
\centering
\caption{\textbf{DLER-1.5B marker reduction relative to length for $CIB_{\beta^{-},1.5B}$.}}
\label{tab:len_vs_mark_red_dler}
\begin{tabular}{lcccc}
\toprule
Metric & n samples & Uniform & $CIB_{\beta^{-},1.5B}$ & \% change \\
\midrule
n\_markers & 1432 & 24.43 & \textbf{9.45} & -61.32\% \\
markers\_per\_1k\_tokens & 1432 & 10.58 & \textbf{5.29} & -49.99\% \\
\bottomrule
\end{tabular}
\end{table}

The mean marker count falls from 24.43 to 9.45, primarily through reduced backtracking. Marker density falls from 10.58 to 5.29 per 1000 tokens, a reduction of approximately 50\%. As in DeepScaler, markers become less frequent relative to length, rather than merely fewer in absolute number.

\subsection{Mechanistic Interpretation}

Across both model families, the semantic prior improves accuracy on every benchmark at comparable average completion length. The paired analyses further show that the number of redundant-reasoning markers declines faster than length, supporting selective compression rather than proportional shortening. These observations are consistent with a content-dependent penalty that suppresses markers associated with redundant reasoning blocks while balancing their removal against accuracy.

Markers remain present, and we observed no reduction in the policy's exploration behavior (Figure~\ref{fig:itc_grid} in Section~\ref{sec:scaling_prop} of the main paper). The $\beta$ parameter in the CIB objective controls the trade-off between compression and accuracy.

\section{Training Details} \label{app:tr_details}
To ensure reproducibility, we provide the full set of hyperparameters and infrastructure details used in our experiments. Our implementation relies on the \texttt{trl} library (version 0.26.2) for Group Relative Policy Optimization (GRPO) and \texttt{lighteval} (version 0.8.1) for robust evaluation.

\subsection{Hyperparameters}
We fine-tuned all models using the hyperparameters listed in Table~\ref{tab:hyperparams}.
\begin{table}[!ht]
\centering
\caption{\textbf{GRPO Training Hyperparameters.} All experiments share these settings unless otherwise noted.}\label{tab:hyperparams}
\begin{tabular}{l|c}
\toprule
\textbf{Hyperparameter} & \textbf{Value} \\ 
\midrule
Optimizer & AdamW \\
Learning Rate & $1 \times 10^{-6}$ \\
LR Scheduler & constant \\
Warmup Ratio & 0.03 \\
Batch Size (Global) & 128 \\
Number of Generations per Prompt & 16 \\
Temperature & 0.8 \\
Max Completion Length & 4096 for DLER and 8192 for Deepscaler  \\
Max Gradient Norm & 1.0 \\
KL Penalty Coefficient ($\beta_{\text{KL}}$) & $5.e^{-4}$ \\
CIB Regularization Weight ($\beta^-, \beta^+$) & $\{5.e^{-5}, 1.5e^{-4}\}$ \\ 
Number of steps & 150 \\ 
\bottomrule
\end{tabular}
\end{table}

\subsection{Evaluation Setup}
We utilized \texttt{lighteval} for all downstream benchmarks.
\begin{itemize}
    \item \textbf{Inference Engine:} vLLM (version 0.10.2).
    \item \textbf{Sampling Strategy:} We used temperature $T=0.6$, top$_p=0.95$, 32K max completion length, and 16 generations per prompt.
    \item \textbf{Hardware:} Training was performed on a node with $8 \times$ NVIDIA H100 (80GB) GPUs.
\end{itemize}

\section{Results from literature}
\label{app:literature_results}
We report additional results from~\cite{wu2025lapo} in Table~\ref{tab:literature_results}.

\begin{table*}[!ht]
\centering

\caption{Compression results for DeepScaler-1.5B~\cite{wu2025lapo}.}
\label{tab:literature_results}

\begin{tabular}{lcc|cc|cc|cc}
\toprule
  \multicolumn{1}{c|}{Model}
 & \multicolumn{2}{c|}{MATH500} 
 & \multicolumn{2}{c|}{AIME24} 
 & \multicolumn{2}{c|}{AMC-23} 
 & \multicolumn{2}{c}{OlympiadBench} \\
\midrule 
 & Acc.$\uparrow$ & Tok.$\downarrow$
 & Acc.$\uparrow$ & Tok.$\downarrow$
 & Acc.$\uparrow$ & Tok.$\downarrow$
 & Acc.$\uparrow$ & Tok.$\downarrow$\\
\midrule
DeepScaler-1.5B  & 85.8 & 3280 & 35.5 & 9246 & 74.2 & 6416 & 54.6 & 5974 \\ \midrule
ThinkPrune-I2k  & 85.5 & 1707 & 34.9 & 5095 & 74.3 & 2913 & 54.7 & 3498 \\
ThinkPrune-4k   & 86.6 & 2042 & 35.5 & 6488 & 76.3 & 3839 & 55.7 & 4010 \\
HAPO            & 84.4 & 2370 & 31.4 & 7702 & 70.3 & 4301 & 51.4 & 4571 \\
AutoThink       & 84.9 & 1635 & 36.2 & 7201 & 67.8 & 3658 & 52.5 & 4085 \\
Thinkless       & 81.3 & 2944 & 28.9 & 9143 & 65.7 & 5276 & 50.2 & 6057 \\
LAPO-D   & 86.4 & 2365 & 37.6 & 5945 & 77.6 & 3655 & 56.1 & 4499 \\
LAPO-I   & 86.3 & 2168 & 38.1 & 5371 & 78.3 & 3765 & 56.3 & 4024 \\
\bottomrule
\end{tabular}

\end{table*}
\section{The list of datasets and models}
\label{app:assets}

\begin{table}[ht]
  \centering
  \caption{The list of datasets.}
  \label{tab:datasets}
  \begin{tabular}{lcc}
    \toprule
    \textbf{Dataset} & \textbf{Link} & \textbf{License} \\
    \midrule
    MATH500~\citep{lightman2023lets}
      & \href{https://huggingface.co/datasets/math-ai/math500}{\texttt{HuggingFace}}
      & MIT \\[2pt]
    AIME 2024~\citep{aime2024}
      & \href{https://huggingface.co/datasets/HuggingFaceH4/aime_2024}{\texttt{HuggingFace}}
      & Copyright MAA \\[2pt]
    AIME 2025~\citep{aime2025}
      & \href{https://huggingface.co/datasets/yentinglin/aime_2025}{\texttt{HuggingFace}}
      & Copyright MAA \\[2pt]
    Minerva Math~\citep{lewkowycz2022solving}
      & \href{https://huggingface.co/datasets/math-ai/minervamath}{\texttt{HuggingFace}}
      & MIT \\[2pt]
    OlympiadBench~\citep{he2024olympiadbench}
      & \href{https://huggingface.co/datasets/math-ai/olympiadbench}{\texttt{HuggingFace}}
      & MIT \\[2pt]
    AMC 2023~\citep{amc2023}
      & \href{https://huggingface.co/datasets/math-ai/amc23}{\texttt{HuggingFace}}
      & Copyright MAA \\[2pt]
    GPQA-Diamond~\citep{rein2023gpqa}
      & \href{https://huggingface.co/datasets/Idavidrein/gpqa}{\texttt{HuggingFace}}
      & CC BY 4.0 \\[2pt]
    MMLU-Pro~\citep{wang2024mmlu}
      & \href{https://huggingface.co/datasets/TIGER-Lab/MMLU-Pro}{\texttt{HuggingFace}}
      & MIT \\[2pt]
    LiveCodeBench~\citep{jain2024livecodebench}
      & \href{https://github.com/LiveCodeBench/LiveCodeBench}{\texttt{GitHub}}
      & MIT \\[2pt]
    DeepScaleR dataset~\citep{luo2025deepscalerdataset}
      & \href{https://huggingface.co/datasets/agentica-org/DeepScaleR-Preview-Dataset}{\texttt{HuggingFace}}
      & MIT \\
    \bottomrule
  \end{tabular}
\end{table}

\begin{table}[ht]
  \centering
  \caption{The list of models.}
  \label{tab:models}
  \begin{tabular}{lcc}
    \toprule
    \textbf{Model} & \textbf{Link} & \textbf{License} \\
    \midrule
    DLER-R1-1.5B~\citep{liu2025dler}
      & \href{https://huggingface.co/nvidia/DLER-R1-1.5B-Research}{\texttt{HuggingFace}}
      & NSCLv1 (research only) \\[2pt]
    DLER-R1-7B~\citep{liu2025dler}
      & \href{https://huggingface.co/nvidia/DLER-R1-7B-Research}{\texttt{HuggingFace}}
      & NSCLv1 (research only) \\[2pt]
    DeepScaleR-1.5B-Preview~\citep{luo2025deepscalerdataset}
      & \href{https://huggingface.co/agentica-org/DeepScaleR-1.5B-Preview}{\texttt{HuggingFace}}
      & MIT \\[2pt]
    Qwen2.5-Base-0.5B~\citep{yang2024qwen2}
      & \href{https://huggingface.co/Qwen/Qwen2.5-0.5B}{\texttt{HuggingFace}}
      & Apache 2.0 \\[2pt]
    Qwen2.5-Base-1.5B~\citep{yang2024qwen2}
      & \href{https://huggingface.co/Qwen/Qwen2.5-1.5B}{\texttt{HuggingFace}}
      & Apache 2.0 \\[2pt]
    Qwen2.5-Base-3B~\citep{yang2024qwen2}
      & \href{https://huggingface.co/Qwen/Qwen2.5-3B}{\texttt{HuggingFace}}
      & Qwen Research \\[2pt]
    Qwen2.5-Base-7B~\citep{yang2024qwen2}
      & \href{https://huggingface.co/Qwen/Qwen2.5-7B}{\texttt{HuggingFace}}
      & Apache 2.0 \\[2pt]
    L3L1-1.5B-Exact~\citep{aggarwal2025l1}
      & \href{https://huggingface.co/l3lab/L1-Qwen-1.5B-Exact}{\texttt{HuggingFace}}
      & MIT \\[2pt]

    OpenReasoning-Nemotron-1.5B~\citep{ahmad2025opencodereasoning,moshkov2025aimo}
      & \href{https://huggingface.co/nvidia/OpenReasoning-Nemotron-1.5B}{\texttt{HuggingFace}}
      & CC-BY-4.0 \\
    \bottomrule
  \end{tabular}
\end{table}

\clearpage



\end{document}